\documentclass{article}

\usepackage{microtype}
\usepackage{graphicx}
\usepackage{subcaption}
\usepackage{booktabs} 

\usepackage{hyperref}

\usepackage[accepted]{icml2026}

\usepackage{amsmath}
\allowdisplaybreaks
\usepackage{amssymb}
\usepackage{mathtools}
\usepackage{amsthm}
\usepackage{algorithmic}
\usepackage{thmtools}

\usepackage{graphicx}
\usepackage{subcaption}

\usepackage[capitalize,noabbrev]{cleveref}

\theoremstyle{plain}
\newtheorem{theorem}{Theorem}[section]
\newtheorem{proposition}[theorem]{Proposition}
\newtheorem{lemma}[theorem]{Lemma}

\theoremstyle{definition}

\newtheorem{assumption}[theorem]{Assumption}
\theoremstyle{remark}

\usepackage[textsize=tiny]{todonotes}

\icmltitlerunning{Boosting CVaR Policy Optimization with Quantile Gradients}

\begin{document}

\twocolumn[
  \icmltitle{Boosting CVaR Policy Optimization with Quantile Gradients}









  \icmlsetsymbol{equal}{*}

  \begin{icmlauthorlist}
    \icmlauthor{Yudong Luo}{hec,mila}
    \icmlauthor{Erick Delage}{hec,mila}
  \end{icmlauthorlist}

  \icmlaffiliation{hec}{GERAD \& Department of Decision Sciences, HEC Montreal, Canada}
  \icmlaffiliation{mila}{Mila - Quebec AI Institute, Canada}

  \icmlcorrespondingauthor{Yudong Luo}{yudong.luo@hec.ca}
  \icmlcorrespondingauthor{Erick Delage}{erick.delage@hec.ca}

  \icmlkeywords{CVaR, quantile}

  \vskip 0.3in
]



\printAffiliationsAndNotice{}  

\begin{abstract}
  Optimizing Conditional Value-at-risk (CVaR) using policy gradient (a.k.a CVaR-PG) faces significant challenges of sample inefficiency. This inefficiency stems from the fact that it focuses on tail-end performance and overlooks many sampled trajectories. We address this problem by augmenting CVaR with an expected quantile term. Quantile optimization admits a dynamic programming formulation that leverages all sampled data, thus improves sample efficiency. This does not alter the CVaR objective since CVaR corresponds to the expectation of quantile over the tail. Empirical results in domains with verifiable risk-averse behavior show that our algorithm within the Markovian policy class substantially improves upon CVaR-PG and consistently outperforms other existing methods.
\end{abstract}
\section{Introduction}
Risk avoidance is a crucial and practical consideration in sequential decision-making which inspires risk-averse reinforcement learning (RARL). Conventional risk-neutral RL~\cite{sutton1998reinforcement} focuses on maximizing the total return's expectation; in contrast, RARL involves optimizing some risk metrics of the return variable. Commonly used risk metrics include variance~\cite{tamar2012policy}, Gini deviation~\cite{luo2023alternative}, Value-at-risk (VaR)~\cite{hau2025q}, conditional VaR (CVaR)~\cite{tamar2015optimizing}, entropic risk measure~\cite{hau2023entropic}, entropic VaR~\cite{su2025risk}. In this work, we focus on (static) CVaR. Intuitively, CVaR emphasizes the worst case value of a policy's return, i.e., the expected value under a specified quantile level $\alpha$ (also called risk level). Thus, optimizing CVaR ensures the prevention of catastrophic outcomes of a policy.

Optimizing static CVaR in RL is a non-trivial task, since it fails to decompose into per-step CVaR optimization in policy improvement~\cite{hau2023dynamic}. Therefore, dynamic programming as in risk-neutral RL cannot be applied directly. A common and practical approach is trajectory-based CVaR policy gradient (PG) using Markovian policies~\cite{tamar2015optimizing,greenberg2022efficient,luo2024simple}. CVaR-PG samples a batch of $N$ trajectories and updates the policy with $\alpha N$ of them with worst returns. This results in sample inefficiency due to 1) $1-\alpha$ portion of sampled data is discarded; 2) good-performing trajectories cannot benefit policy learning, termed  blindness to success by~\citet{greenberg2022efficient}. In cases where return distribution is discrete, CVaR-PG may even suffer from vanishing gradients~\cite{greenberg2022efficient}. From its variational representation~\cite{rockafellar2000optimization}, optimizing CVaR is a bi-level problem, where the inner policy learning reduces to a risk-neutral case~\cite{bauerle2011markov}. Based on this idea, several recent works aim to improve sample efficiency by assigning per-step rewards such that the expected total return of the trajectory aligns with the CVaR objective. In consequence, all trajectories are used for policy learning in a risk-neutral style, thus improving sample efficiency.


This work aims to improve sample efficiency and accelerate policy learning of CVaR-PG from another perspective. We propose augmenting it with quantile PG of tail returns. This augmentation is reasonable and intuitive, since CVaR can be regarded as the expectation of tail quantiles. The benefits of this approach are twofold. 1) Different from static CVaR, optimizing static quantile enjoys a nested form where dynamic programming is available~\cite{hau2025q}. Thus all data can be used for policy learning, which improves sample efficiency. 2) Good-performing trajectories are captured by the quantile PG component, which in turn overcomes the blindness to success of CVaR-PG.

We first derive a new Bellman operator for static VaR (quantile), motivated by the framework of~\citet{hau2025q}, from which an actor-critic algorithm with stochastic policy can be readily obtained. To operate within the Markovian policy class, we then show the optimal VaR policy can be recovered from a quantile value function associated with a Markovian policy by tracking cumulative rewards, motivating a modified actor-critic algorithm for learning Markovian VaR policy. Finally, we incorporate VaR-PG with CVaR-PG to form the final algorithm. Empirical results in domains with verifiable risk-averse behavior demonstrate that our method significantly improves the sample efficiency of CVaR-PG, and can learn risk-averse policy with high return and low risk in terms of CVaR when others may fail.

\section{Background}
\subsection{Quantile, VaR and CVaR}
For a real-valued random variable $\tilde{x}$ (random variables are marked with a tilde, e.g. $\tilde{x}$, in this paper), its quantile at level $\alpha\in[0,1]$ is any $x\in \mathbb{R}$ such that $P[\tilde{x}\leq x]\geq \alpha$ and $P[\tilde{x}\geq x]\geq 1-\alpha$. It may not be unique and lies in the interval $[q^-_\alpha(\tilde{x}), q^+_\alpha(\tilde{x})]$ where $q^-_\alpha(\tilde{x}):=\min\{x\in\mathbb{R}|P[\tilde{x}\leq x]\geq \alpha\}$ and $q^+_\alpha(\tilde{x}):=\max\{x\in\mathbb{R}|P[\tilde{x}< x]\le \alpha\}$.

Given a risk-level $\alpha\in[0,1]$, $\alpha$-VaR is defined as the largest $1-\alpha$ confidence lower bound on the value of $\tilde{x}$, i.e.,
\begin{equation*}
    \mathrm{VaR}_\alpha[\tilde{x}]:=q^+_\alpha(\tilde{x})
\end{equation*}
For ease of presentation, we consider $\tilde{x}$ is continuous and $q^-_\alpha(\tilde{x})=q^+_\alpha(\tilde{x})$ in this paper. In this case, VaR is equivalent to quantile.

The left tail CVaR at risk level $\alpha$ is defined as 
\begin{equation}
\label{eq:cvar-def-1}
    \mathrm{CVaR}_\alpha[\tilde{x}]:=\frac{1}{\alpha}\int_0^\alpha \mathrm{VaR}_\beta[\tilde{x}] d\beta~.
\end{equation}
Thus, $\mathrm{CVaR}_\alpha[\tilde{x}]$ can be interpreted as the expected value of the bottom $\alpha$ portion of tail quantiles, i.e., $\mathrm{CVaR}_\alpha[\tilde{x}]=\mathbb{E}[\tilde{x}|\tilde{x}\leq \mathrm{VaR}_\alpha[\tilde{x}]]$. CVaR has a variational representation as~\cite{rockafellar2000optimization}
\begin{equation}
\label{eq:cvar-def-2}
    \mathrm{CVaR}_\alpha[\tilde{x}]:=\max_{y\in\mathbb{R}} y -\frac{1}{\alpha}\mathbb{E}[(y-\tilde{x})^+],
\end{equation}
where $(x)^+:=\max(x,0)$. The maximum is always attained at $y^*=\mathrm{VaR}_\alpha[\tilde{x}]$ as a by product.
\subsection{CVaR Optimization in RL}
In standard RL settings, the agent-environment interaction is modeled as a Markov decision process (MDP), represented as a tuple $\langle \mathcal{S},\mathcal{A},R,P,\mu_0,\gamma\rangle$. $\mathcal{S}$ and $\mathcal{A}$ denote state and action spaces. $R(s,a,\tilde{\varepsilon})$ is the state and action dependent reward function with randomness $\tilde{\varepsilon}$. $P(\cdot|s,a)$ is the transition function. $\mu_0$ is the initial state distribution and $\gamma\in(0,1]$ is the discount factor. For notational convenience, we use $\tilde{r}(s,a)$ to denote the random reward variable associated with the pair $(s,a)$, and $r(s,a)$ to denote a realized reward. We may omit the arguments and simply use $\tilde{r}$ and $r$ when the context is clear. Agent collects trajectories, i.e., a sequence of state $\tilde{s}_t$, action $\tilde{a}_t$, reward $\tilde{r}_t$, by executing its policy $\pi$. Define the future rerturn variable at time $t$ as $\tilde{g}_t:=\sum_{i=0}^{T-1-t}\gamma^i \tilde{r}(\tilde{s}_{t+i}, \tilde{a}_{t+i})$ (it is possible  that $T=\infty$), thus $\tilde{g}_0$ indicates the total return variable starting from the initial state. Let $\Pi_{\mathcal{HD}}$ be the set of history-dependent policies and $\Pi_{\mathcal{M}}$ be the set of Markovian policies.

When considering static CVaR at risk level $\alpha$, agent aims to find a policy that maximizes the following objective
\begin{equation*}
    \max_{\pi\in\Pi_{\mathcal{HD}}} \mathrm{CVaR}^{\pi}_\alpha[\tilde{g}_0].
\end{equation*}
Generally, the optimal $\alpha$-CVaR policy is history dependent~\cite{lim2022distributional}. Consider Eq.~\ref{eq:cvar-def-2} when $y$ is fixed, the policy optimization remains $\max_\pi-\mathbb{E}[(y-\tilde{g}_0)^+]$. A common approach is to formulate an augmented MDP with new state $(s,k)\in \mathcal{S}\times\mathbb{R}$, where $k$ is a moving variable keeping track of the rewards collected so far, i.e., $k_t:=\sum_{i=0}^{t-1}\gamma^i r_i$~\cite{bauerle2011markov}. In this augmented MDP, the per-step reward is $0$ except the terminal state that gives a reward $-(y-k_T)^+$. Therefore, the reward is sparse. 

In practice, we may seek to learn Markovian policy given its simplicity. Prior works focusing on Markovain CVaR policy have validated its practical performance~\cite{tamar2015optimizing,greenberg2022efficient,luo2024simple,kim2024risk,mead2025return}. In addition, \citet{lim2022distributional} showed that if a CVaR-optimal policy is Markovian in original MDP, the optimal history-dependent policy can be recovered from the quantile value function of the Markovian policy, indicating the inter-connection between Markovian and history-dependent policies. Consider $\bar{\pi}\in\Pi_{\mathcal{M}}$ parameterized by $\theta$. Under some mild assumptions, CVaR-PG can be derived by taking derivative of Eq.~\ref{eq:cvar-def-1} and is estimated by sampling $N$ trajectories $\{\tau_i\}_{i=1}^N$ using $\bar{\pi}_\theta$~\cite{tamar2015optimizing}, i.e.,
\begin{equation}
\label{eq:cvar-pg}
\begin{aligned}
    \nabla_\theta \mathrm{CVaR}_\alpha[\tilde{g}_0]\simeq \frac{1}{\alpha N} \sum_{i=1}^N\Big (&\mathbb{I}_{\{R(\tau_i)\leq \hat{q}_\alpha\}} (R(\tau_i)-\hat{q}_\alpha)\\ &\sum_{t=0}^{T-1} \nabla_\theta \log\bar{\pi}_\theta(a_{i,t}|s_{i,t})\Big ),
\end{aligned}
\end{equation}
where $R(\tau_i):=\sum_t\gamma^t r_{i,t}$ is the total return of trajectory $\tau_i$, $T$ is the trajectory length, and $\hat{q}_\alpha$ is the empirical $\alpha$-quantile estimated from $\{R(\tau_i)\}_{i=1}^N$. Similar Markovian PG can also be derived from Eq.~\ref{eq:cvar-def-2} with $y$ fixed to $\hat{q}_\alpha$.

\paragraph{Limitations of CVaR-PG.} The CVaR-PG estimator in Eq.~\ref{eq:cvar-pg} suffers from low sample efficiency. First, only $\alpha$ portion of sampled trajectories contributes to gradient estimation, while the remainder is discarded. Second, the contributing trajectories correspond exclusively to the worst-performing outcomes, which can significantly slow learning, a phenomenon termed \emph{blindness to success} by \citet{greenberg2022efficient}. In addition, the term $\mathbb{I}_{\{R(\tau_i)\leq \hat{q}_\alpha\}} (R(\tau_i)-\hat{q}_\alpha)$ can equal zero if the left tail of the return distribution is overly flat, in which case $R(\tau_i)=\hat{q}_\alpha$ for any $\tau_i$ satisfying $R(\tau_i)\leq \hat{q}_\alpha$. This gradient vanishing phenomenon is often overlooked when assuming continuous rewards.

Several approaches have been proposed to improve the sample efficiency of CVaR-PG. \citet{greenberg2022efficient} introduced a cross-entropy–based sampling strategy that biases data collection toward high-risk scenarios, but it requires partial or full control over environment stochasticity. \citet{luo2024simple} proposed a policy-mixing approach that combines a risk-neutral policy with an adjustable policy to construct a risk-averse policy, which is effective primarily when optimal risk-neutral and risk-averse actions coincide over most of the state space. Another line of work reformulates the trajectory-level return, i.e., $\mathbb{I}_{\{R(\tau_i)\leq \hat{q}_\alpha\}} R(\tau_i)$, into per-step rewards, allowing all trajectories to contribute to policy updates in a risk-neutral manner. For instance, \citet{kim2024risk} proposed Predictive CVaR-PG (PCVaR-PG), which learns a predictor $f(s,k)=\mathbb{P}(R(\tau)\leq \hat{q}_\alpha)$ and reweights rewards as $\hat{r}_t = f(s_t,k_t) r_t$. Similarly, \citet{mead2025return} proposed return capping (RET-CAP) which reformulated $\max_{\bar\pi} -\mathbb{E}[(q^*_\alpha-R(\tau))^+]$ (cf. Eq.~\ref{eq:cvar-def-2} when set y as the optimal $\alpha$-quantile $q^*_\alpha$) as $\max_{\bar\pi}\mathbb{E}[(R(\tau)-q^*_\alpha)^-]=\max_{\bar\pi}\mathbb{E}[\min(R(\tau),q^*_\alpha)-q^*_\alpha]$. $q^*_\alpha$ is a constant and can be omitted. It then decomposed $\min(R(\tau),q^*_\alpha)$ to per-step reward $\hat{r}_t = \min(k_t, q^*_\alpha) - \min(k_{t-1}, q^*_\alpha)$. All $k$ parameters here are cumulative reward tracking variables as in \citet{bauerle2011markov}.

\subsection{VaR Optimization in RL}
\label{sec:var-rl}
When considering static VaR at risk level $\alpha$, the objective to solve is
\begin{equation*}
    \max_{\pi\in\Pi_{\mathcal{HD}}}\mathrm{VaR}^{\pi}_\alpha[\tilde{g}_0].
\end{equation*}
The optimal $\alpha$-VaR policy is generally also history-dependent. However, if considering Markovian policy, a policy gradient for VaR analogous to Eq.~\ref{eq:cvar-pg} can be derived (see, e.g., \citet{jiang2022quantile}), in which only  $\alpha$-portion of sampled trajectories contributes to the gradient. 

Recently, \citet{li2022quantile} and \citet{hau2023dynamic} showed that the static VaR admits a dynamic decomposition for policy optimization, and the value function satisfies a Bellman-like equation. However, the equation in \citet{li2022quantile} is model-based where a constrained optimization problem involving the transition probability need to be solved. To address this limitation, \citet{hau2025q} proposed a nested VaR Bellman equation and developed a model-free $Q$-learning-style algorithm for value and policy learning.

Define the optimal $\alpha$-quantile value obtained by taking action $a$ at state $s$ as
\vspace{-0.05in}
\begin{equation*}
    q^*(s,\alpha,a):=\max_{\pi\in\Pi_{\mathcal{HD}}}\mathrm{VaR}^\pi_\alpha[\sum_{t=0}^{\infty}\tilde{r}(\tilde{s}_t,\tilde{a}_t)|\tilde{s}_0=s,\tilde{a}_0=a].
    \vspace{-0.05in}
\end{equation*}
It is also convenient to define the optimal state quantile value function (it holds that $v^*(s,\alpha)=\max_a q^*(s,\alpha,a)$)
\vspace{-0.05in}
\begin{equation*}
    v^*(s,\alpha):=\max_{\pi\in\Pi_{\mathcal{HD}}}\mathrm{VaR}_\alpha^\pi[\sum_{t=0}^{\infty}\tilde{r}(\tilde{s}_t,\tilde{a}_t)|\tilde{s}_0=s],
    \vspace{-0.05in}
\end{equation*}
and the optimal intermediate policy function $\hat{\pi}(s,\alpha)$ (we call $\hat{\pi}$ intermediate because $\hat{\pi}$ along with $v$ forms $\pi\in\Pi_{\mathcal{HD}}$ in \citet{hau2025q})
\vspace{-0.03in}
\begin{equation*}
    \hat{\pi}^*(s,\alpha):=\arg
    \max_a q^*(s,\alpha,a).
    \vspace{-0.03in}
\end{equation*}
The nested VaR Bellman optimality equation is\footnote{Original equation in \citet{hau2025q} considers deterministic reward $r(s,a)$. We consider random reward $\tilde{r}(s,a)$ and this does not alter the equation. See analysis in Appendix~\ref{sec:eq-random-reward}.}
\vspace{-0.05in}
\begin{equation}
\label{eq:qdp-nested}
    q(s,\alpha,a)=\mathrm{VaR}_\alpha[\tilde{r}(s,a) + \gamma \max_{a'} q(\tilde{s}',\tilde{u},a')]
    \vspace{-0.05in}
\end{equation}
with unique fixed point $q^*$ (see Appendix \ref{app:VaRinfHor}), where $\tilde{u}\sim U[0,1]$ is an independent random variable uniformly distributed on $[0,1]$, and $\mathrm{VaR}_\alpha$ is taken with respect to the joint distribution of $\tilde{r}$, $\tilde{s}'$ and $\tilde{u}$. Eq.~\ref{eq:qdp-nested} is not directly amenable to a $Q$-learning–style algorithm, since the VaR operator is generally unavailable in closed form. By exploiting the elicitability of quantile, i.e., $\alpha$-quantile minimizes the quantile regression loss $l_\alpha(\cdot)$ as
\begin{equation}
\label{eq:quantile_refression}
    \begin{aligned}
    q_\alpha(\tilde{x})&=\arg\min_{y\in\mathbb{R}}\mathbb{E}[l_\alpha(\tilde{x}-y)]\\
    l_\alpha(x-y) &:= (\alpha-\mathbb{I}\{x<y\})(x-y),
    \end{aligned}
\end{equation}
when quantiles are unique, Eq.~\ref{eq:qdp-nested} can therefore be equivalently expressed as
\begin{equation*}
    q(s,\alpha,a)=\arg\min _x \mathbb{E}\Big[l_\alpha\big(\tilde{r}+\gamma\max_{a'}q(\tilde{s}',\tilde{u},a')-x\big)\Big].
\end{equation*}
This leads to a gradient-based
$Q$-learning–style update rule, based on transition $(s,\alpha,a,r,s')$ with learning rate $\chi$, as
\begin{equation}
\label{eq:quantile-q-learning}
\begin{aligned}
q(s,\alpha,a)&\leftarrow q(s,\alpha,a)\\
&+\chi\cdot \mathbb{E}\Big[\partial l_\alpha\big(r+\gamma \max_{a'}q(s',\tilde{u},a')-q(s,\alpha,a)\big)\Big].
\end{aligned}
\end{equation}
In practice, \citet{hau2025q} discretized $\alpha$ in a range $[\epsilon,1-\epsilon]$ (with a small $\epsilon \in(0,\frac12)$) and used a soft version of $l_\alpha(\cdot)$, i.e., $l^\kappa_\alpha(\cdot)$ with a parameter $\kappa\in(0,1]$, to make sure $\partial l^\kappa_\alpha(\cdot)$ always exits and for convergence guarantees. The derivative of $l^\kappa_\alpha(\cdot)$ used in \citet{hau2025q} is
\begin{equation}
\label{eq:partial-soft}
    \partial l_\alpha^\kappa(\delta)=
 \left\{
\begin{alignedat}{2}
 &(1-\alpha)(\kappa \delta+\kappa^2-1) &&~\text{if}~\delta<-\kappa \\
&\frac{1-\alpha}{\kappa} \delta &&~\text{if}~\delta\in[-\kappa,0)\\
& \frac{\alpha}{\kappa} \delta &&~\text{if}~\delta\in[0,\kappa)\\
& \alpha(\kappa \delta - \kappa^2 + 1)&&~\text{if}~\delta\geq\kappa.
\end{alignedat}
\right.
\end{equation}
This method differs from risk-neutral quantile-based distributional RL~\cite{dabney2018distributional} in that it selects the optimal action separately for each quantile level in Eq.~\ref{eq:qdp-nested}, rather than optimizing the expectation over the return distribution. Alternatively, this method can be interpreted as learning the optimal value on an augmented MDP with new state $(s,\alpha)\in \mathcal{S}\times [0,1]$ and state-dependent risk-aversion. 

Note that $\hat{\pi}^*$ alone is insufficient to recover the optimal VaR policy, since we also need to determine the $\alpha$-parameter at each state. To implement the optimal (history-dependent) VaR policy, the execution procedure will track the cumulative reward so far and compare it with optimal quantile value to determine, at each state, the appropriate quantile level to select the corresponding optimal action (cf. Algo.~\ref{alg:static-var-exec}).

\begin{algorithm}[tb]
  \caption{VaR Actor-Critic}
  \label{alg:var-pi}
  \begin{algorithmic}
    \STATE {\bfseries Input:} quantile value function $v(s,\alpha)$, policy function $\hat{\pi}(\cdot|s,\alpha)$, step size $\eta$, iteration $M$
    \FOR{$i=1$ {\bfseries to} $M$}
    \FOR{$(s,\alpha)\in \mathcal{S}\times[0,1]$}
    \STATE Sample $a\sim \hat{\pi}(\cdot|s,\alpha)$
    \STATE Update $v(s,\alpha)$ by Eq.~\ref{eq:var-pe}
    \ENDFOR
    \FOR{$(s,\alpha)\in \mathcal{S}\times[0,1]$}
    \STATE Compute $A(s,\alpha,a)$  by Eq.~\ref{eq:adv} for $a$ sampled before
    \STATE Update $\hat{\pi}(\cdot|s,\alpha)$ by $A(s,\alpha,a)\nabla\log\hat{\pi}(a|s,\alpha)$
    \ENDFOR
    \ENDFOR
  \end{algorithmic}
\end{algorithm}

\section{Boosting CVaR Optimization with VaR}
As discussed in Sec.~\ref{sec:var-rl}, optimizing static VaR enjoys a dynamic programming style update where all data is used for policy learning. This motivates \textcolor{black}{the idea of reformulating CVaR optimization as a weighted sum of its two equivalent representations in Eq. \ref{eq:cvar-def-1} and \ref{eq:cvar-def-2} to improve sample efficiency and accelerate learning.} Equivalence can be verified in definition 4.43 and Lemma 4.46 of \citet{follmer_stochastic_2016}.
Therefore, we consider the following objective
\begin{equation}
\label{eq:obj}
    \max_{\bar\pi\in\Pi_{\mathcal{M}}} \textcolor{black}{(1-\omega)}\mathrm{CVaR}_\alpha[\tilde{g}_0] + \omega\mathbb{E}_{\beta\sim U[0,\alpha]}\big[\mathrm{VaR}_\beta[\tilde{g}_0]\big],
\end{equation}
where $\omega\in(0,1)$ is a trade-off parameter.

Following prior work, we also focus on learning Markovian policies. For CVaR, optimization within the Markovian policy class has been widely studied. For VaR, although acquiring a VaR-optimal policy is discussed in \citet{hau2025q}, the resulting policy depends jointly on $\hat{\pi}(s,\alpha)$ and $v(s,\alpha)$, which can be challenging to learn reliably under function approximation. Therefore, we aim to develop proximal VaR-PG using Markovian policies. Motivated by \citet{hau2025q}, we first derive a new Bellman operator for the state-only quantile value function, from which an actor-critic algorithm with stochastic policy can be readily obtained. We then show how this VaR algorithm can be modified when restricting to the Markovian policy class.



\subsection{The New VaR Bellman Operator}
\label{sec:new-bellman}
\textcolor{black}{Though the Q-learning algorithm in \citet{hau2025q} offers a straightforward way for performing policy gradients with stochastic polices, i.e., $q(s,\alpha,a)\nabla\log\pi(a|s,\alpha)$, in practice, state-of-the-art on policy policy gradient algorithms tends to avoid directly using the action value due to its high variance, e.g., PPO~\cite{schulman2017proximal}, therefore, we aim to derive a new Bellman equation for state only value function that are better suited for policy updates.} Motivated by the elicitability of quantile and following the soft quantile loss function in \citet{hau2025q}, the new VaR Bellman optimality operator we consider is defined as
\begin{equation}
\label{eq:var-v-bellman}
\begin{aligned}
    &\text{\small$\mathcal{T}_{\epsilon,\kappa}^*v(s,\alpha)=v(s,\alpha)+\eta\max_a \mathbb{E}[\partial l^\kappa_{\epsilon,\alpha}(\delta_v(s,\alpha,\tilde{r}(s,a),\tilde{s}',\tilde{u}))]$}\\
    &\delta_v(s,\alpha,r,s',u):=r+\gamma v(s',u)-v(s,\alpha),
\end{aligned}
\end{equation}
where $\eta$ is the step size, $\partial l_{\epsilon,\alpha}^\kappa(\cdot):=\partial l^\kappa_{\max(\epsilon,\min(1-\epsilon,\alpha))}(\cdot)$ clips the $\alpha$ of $\partial l^\kappa_{\alpha}(\cdot)$ as defined in Eq.~\ref{eq:partial-soft} with a small $\epsilon\in(0,\frac12)$, $\delta_v$ is analogous to an one-step TD error. $\tilde{u}\sim U[0,1]$ is an independent random variable. Note that when discretizing the quantile levels, e.g., as in QR-DQN~\cite{dabney2018distributional}, $\alpha$ is guaranteed to lie in $[\epsilon,1-\epsilon]$ so the clipping of $\alpha$ is never applied. \textcolor{black}{Note that one cannot  obtain Eq.~\ref{eq:var-v-bellman} by simply taking the maximum over $a$ of Eq.~\ref{eq:quantile-q-learning}, since $l_\alpha(\cdot)$ is non-linear. }

\begin{restatable}{proposition}{Propcontraction}
\label{prop:contraction}
     $\mathcal{T}_{\epsilon,\kappa}^*$ is a contraction mapping for $v$  with step size $\eta\in(0,\kappa]$.
\end{restatable}

\begin{restatable}{proposition}{Propvstar}
\label{prop:fix-point}
    The optimal quantile value $v^*(s,\alpha)$ is the unique fixed point of $\mathcal{T}_{\epsilon,\kappa}^*$ when $\epsilon=0$ and $\kappa= 0$.
\end{restatable}

\textbf{Remark}. 
Similar to \citet{hau2025q}, using $l^\kappa_{\epsilon,\alpha}(\cdot)$ in Eq.~\ref{eq:var-v-bellman} is for better theoretical analysis. In practice, for easier computation, one can use $l_\alpha(\cdot)$ in Eq.~\ref{eq:quantile_refression} as an extreme and easy-to-use case. This gives the Bellman operator
\begin{small}
$$
\mathcal{T}^*v(s,\alpha)=v(s,\alpha)+\eta\max_a \mathbb{E}[\alpha-\mathbb{I}\{\delta_v(s,\alpha,\tilde{r}(s,a),\tilde{s}',\tilde{u})<0\}].
$$
\end{small}

It is straightforward to develop an actor-critic algorithm from this Bellman optimality operator by using stochastic policy. Consider the VaR Bellman operator under $\hat{\pi}$ as
\begin{small}
\begin{equation}
\label{eq:var-pe}
    \mathcal{T}^{\hat{\pi}}_{\epsilon,\kappa} v(s,\alpha) = v(s,\alpha) + \eta \mathbb{E}_{a\sim \hat{\pi}(s,\alpha)}\mathbb{E}[\partial l^\kappa_{\epsilon,\alpha}(\delta_v(s,\alpha,\tilde{r}(s,a),\tilde{s}',\tilde{u}))],
\end{equation}
\end{small}
and define the advantage function as
\begin{equation}
\label{eq:adv}
    A(s,\alpha,a)=\mathbb{E}[\partial l^\kappa_{\epsilon,\alpha}(\delta_v(s,\alpha,\tilde{r}(s,a),\tilde{s}',\tilde{u}))],
\end{equation}
the policy can be updated towards the increasing direction of $A(s,\alpha,a)$ by the gradient $A(s,\alpha,a) \nabla \log \hat{\pi}(a|s,\alpha)$. \textcolor{black}{Similar to Propositions~\ref{prop:contraction} and \ref{prop:fix-point}, we show in Appendix~\ref{app:T-pi} that $\mathcal{T}^{\hat{\pi}}_{\epsilon,\kappa}$ in Eq.~\ref{eq:var-pe} is a contraction mapping, and provide its fixed point when $\epsilon=0$ and $\kappa=0$. For policy gradient using advantage in Eq.~\ref{eq:adv}, we show in Appendix~\ref{app:var-pg} that if the space of  $\mathcal{A}$ is finite, and the policy parameterization is softmax, it guarantees that the fixed point of Eq.~\ref{eq:var-pe} is monotone increasing pointwise. With the above analysis,} the VaR actor-critic algorithm (assuming the MDP is known) is described in Algo.~\ref{alg:var-pi}. Note that the expectation in Eq.~\ref{eq:var-pe} and \ref{eq:adv} is taken with respect to the joint distribution of $\tilde{r}$, $\tilde{s}'$, and $\tilde{u}$. In general on-policy policy updates, the agent is only able to sample a single transition instance $(s,\alpha,a,r,s')$ along the trajectory. This still gives unbiased stochastic updates for both value and policy functions. 

\begin{algorithm}[tb]
  \caption{Static VaR Policy Execution~\cite{hau2025q}}
  \label{alg:static-var-exec}
  \begin{algorithmic}
    \STATE {\bfseries Input:} quantile value function $v^*(s,\alpha)$, policy function $\hat{\pi}^*(s,\alpha)$, $s_0\in \mathcal{S}$, $\alpha_0\in[0,1]$, $\gamma\in(0,1]$
    \STATE $s \leftarrow s_0, \alpha\leftarrow \alpha_0$
    \FOR{$t=1$ {\bfseries to} $T$}
    \STATE $a \leftarrow \hat{\pi}^*(s,\alpha), ~~z \leftarrow v^*(s,\alpha)$
    \STATE $r, s' = \mathrm{env.step}(a)$
    \STATE $z\leftarrow (z - r) / \gamma$
    \STATE $\alpha \leftarrow \min\{\beta|v(s',\beta)\geq z\}$
    \STATE $s\leftarrow s'$
    \ENDFOR
  \end{algorithmic}
\end{algorithm}


\subsection{Reducing to Markovian Policy Class}
\label{sec:mkv-policy}
As suggested by \citet{hau2025q}, given $v^*$ and $\hat{\pi}^*$, the optimal (history-dependent) VaR policy is constructed by Algo.~\ref{alg:static-var-exec}. Specifically, the agent tracks the cumulative reward accrued up to time $t$ and compare it with $v^*(s_t,\tilde{u})$ to determine the corresponding quantile level $\alpha_t$, at which point it executes $\hat{\pi}^*(s_t,\alpha_t)$. Though this algorithm is validated in tabular settings~\cite{hau2025q}, it can be challenging when using function approximation. Thus, we seek to find proximal VaR-PG in the Markovian policy class.

Let $\bar{\pi}^*\in\Pi_{\mathcal{M}}$ be a stationary and Markovian  $\alpha_0$-VaR optimal policy. By definition of stationary and Markovian, there exists $\hat\pi^*$ such that $\bar\pi^*(s)=\hat{\pi}^*(s,\alpha)$ for all $(s,\alpha)$ pairs encountered when starting from $(s_0,\alpha_0)$ and executing Algo.~\ref{alg:static-var-exec}. Denote the quantile value function under $\bar\pi^*$ as $q^{\bar\pi^*}(s,\alpha,a)$. Obtaining $q^{\bar\pi^*}$ is the same as doing policy evaluation in distributional RL. Here we show that the optimal $\alpha_0$-VaR policy can still be recovered by tracking cumulative rewards as in Algo.~\ref{alg:static-var-exec} but using $q^{\bar{\pi}^*}$ instead of $q^*$. The new action execution using $q^{\bar{\pi}^*}$ is described in Algo~\ref{alg:static-var-exec1}.

\begin{algorithm}[tb]
  \caption{Static VaR Policy Execution with $q^{\bar\pi^*}$}
  \label{alg:static-var-exec1}
  \begin{algorithmic}
    \STATE {\bfseries Input:} quantile value function $q^{\bar\pi^*}(s,\alpha,a)$,  $s_0\in \mathcal{S}$, $\alpha_0\in[0,1]$, $\gamma\in(0,1]$
    \STATE $s \leftarrow s_0, \alpha\leftarrow \alpha_0$
    \FOR{$t=1$ {\bfseries to} $T$}
    \STATE $a\leftarrow \arg\max_b q^{\bar\pi^*}(s,\alpha,b),~~z\leftarrow q^{\bar\pi^*}(s,\alpha,a)$
    \STATE $r, s' = \mathrm{env.step}(a)$
    \STATE $z\leftarrow (z-r)/\gamma$
    \STATE $\alpha\leftarrow \min\{\beta|\max_b q^{\bar\pi^*}(s',\beta,b)\geq z\}$
    \STATE $s\leftarrow s'$
    \ENDFOR
  \end{algorithmic}
\end{algorithm}
\begin{restatable}{proposition}{Propquantpibar}
\label{prop:quant-pibar}
    Let $\bar\pi^*\in \Pi_{\mathcal{M}}$ be a stationary and Markovian $\alpha_0$-VaR optimal policy. Assume that $\bar\pi^*$ is unique. Running Algo.~\ref{alg:static-var-exec1} with $\bar\pi^*$'s quantile value function $q^{\bar\pi^*}$ results in executing $\bar\pi^*$.
\end{restatable}


\begin{proof}
    We show that for every state $s_t$, Algo.~\ref{alg:static-var-exec1} executes $\bar\pi^*(s_t)$. The key to proof is to show that $q^{\bar\pi^*}$ satisfies the following property. Suppose running Algo.~\ref{alg:static-var-exec} with $v^*$ and $\hat{\pi}^*$ results in a sequence of $(s_t,\alpha_t)$. Then $q^{\bar\pi^*}(s_t,\alpha_t,\bar\pi^*(s_t))=v^*(s_t,\alpha_t)$. Recall that by definition $\bar\pi^*(s_t)=\hat{\pi}^*(s_t,\alpha_t)$.

    First, consider the value of $q^{\bar\pi^*}(s_t,\alpha_t,\bar\pi^*(s_t))$. Since $v^*(s_t,\alpha_t)$ is the optimal quantile value in the MDP, $q^{\bar\pi^*}(s_t,\alpha_t,\bar\pi^*(s_t))\leq v^*(s_t,\alpha_t)$. If $q^{\bar\pi^*}(s_t,\alpha_t,\bar\pi^*(s_t))< v^*(s_t,\alpha_t)$, then there exists another action $b\neq \bar\pi^*(s_t)$ that can achieve higher values for $v(s_t,\alpha_t)$. This conflicts with the fact that $\bar\pi^*(s_t) =\hat{\pi}^*(s_t,\alpha_t)$ is already the action that achieves $v^*(s_t,\alpha_t)$ in the MDP for all $(s_t,\alpha_t)$. Therefore, $q^{\bar\pi^*}(s_t,\alpha_t,\bar\pi^*(s_t))=v^*(s_t,\alpha_t)$.
    
    Second, consider the action execution procedure. For $s_0$, since $\bar{\pi}^*$ is the optimal $\alpha_0$-VaR policy, $\bar\pi^*(s_0)$ must be selected, and $q^{\bar\pi^*}(s_0,\alpha_0,\bar{\pi}^*(s_0))=v^*(s_0,\alpha_0)$. For $s_1$, since $\max_bq^{\bar\pi^*}(s_1,\beta,b)$ is monotonic with respect to $\beta$, the alpha selection procedure will still choose $\alpha_1=\min\{\beta|\max_b q^{\bar{\pi}^*}(s_1,\beta,b)\geq \frac{v^*(s_0,\alpha_0)-r_0}{\gamma}\}$ as in Algo.~\ref{alg:static-var-exec}, because $\max_b q^{\bar{\pi}^*}(s_1,\alpha_1,b)=v^*(s_1,\alpha_1)$. Thus $\arg\max_b q^{\bar{\pi}^*}(s_1,\alpha_1,b)=\bar{\pi}^*(s_1)$ is selected. The same applies for all remaining $s_t$.

    \textcolor{black}{Note that assuming the uniqueness of $\bar{\pi}^*$  is for ease of presentation. Otherwise, the action section in Algo.~\ref{alg:static-var-exec1} needs to be modified to ensure that the same action selection rule, among optimal actions, is applied.}
\end{proof}

Proposition~\ref{prop:quant-pibar} establishes the fact that one can execute a VaR-optimal policy from a quantile value function induced by a Markovian policy by tracking the cumulative rewards in the MDP. This action execution algorithm also suggests the following optimization objective for $\bar{\pi}^*$.
\begin{restatable}{proposition}{Propmkv}
\label{prop:mkv}
Let $v^{\bar{\pi}^*}$ be the state quantile value function of policy $\bar{\pi}^*$ as defined in Proposition~\ref{prop:quant-pibar}, then
\begin{equation}
\label{eq:maximum-barpi}
\begin{aligned}
&\text{\small$\bar{\pi}^*=\arg\max_{\bar{\pi}\in \Pi_\mathcal{M}} \mathbb{E}_{s,\alpha,a\sim\bar\pi(s)}[\partial l^\kappa_{\epsilon,\alpha}(\delta_{v^{\bar{\pi}^*}}(s,\alpha,\tilde{r}(s,a),\tilde{s}',\tilde{u}))]$}\\
&\delta_{v^{\bar{\pi}^*}}(s,\alpha,r,s',u)=r+\gamma v^{\bar{\pi}^*}(s',u)-v^{\bar{\pi}^*}(s,\alpha),
\end{aligned}
\end{equation}
with $(s,\alpha)$ drawn from the distribution of state-risk level pair as in Algo.~\ref{alg:static-var-exec1}, and when $\epsilon=0$ and $\kappa= 0$.
\end{restatable}

The objective in Proposition~\ref{prop:mkv} exhibits a form analogous to the advantage function in Algo.~\ref{alg:var-pi}, suggesting an alternative approach for learning VaR policy within the Markovian policy class. Rather than learning an $\alpha$-dependent policy $\hat{\pi}(s,\alpha)$ via advantage $A(s,\alpha,a)$ across all $\alpha$ levels as in Algo.~\ref{alg:var-pi}, we instead learn a Markovian policy $\bar{\pi}(s)$ using advantage $\bar{A}(s,\alpha,a)$, where the quantile level $\alpha$ is determined by tracking rewards as in Algo.~\ref{alg:static-var-exec1}. Specifically, given a batch of data $\{(s_t,a_t,r_t,s_{t+1})\}_{t=0}^{T-1}$ collected by $\bar{\pi}$, we compute the risk-levels $\alpha_t$ for each state $s_t$ as ($\alpha_0$ is given)
\begin{equation}
\label{eq:track-alpha}
   \alpha_{t}=\min\{\beta|v(s_{t},\beta)\geq \frac{v(s_{t-1},\alpha_{t-1})-r_{t-1}}{\gamma}\}.
\end{equation}
Define the advantage function of the Markovian policy as
\begin{equation}
\label{eq:mkv-adv}
    \bar A(s_t,\alpha_t,a_t) = \mathbb{E}[\partial l^{\kappa}_{\epsilon,\alpha_t}(\delta_{v}(s_t,\alpha_t,r_t,s_{t+1},\tilde{u}))],
\end{equation}
$\bar{\pi}$ is updated by the gradient $\bar A(s_t,\alpha_t,a_t)\nabla \log \bar{\pi}(a_t|s_t)$. $v$ is updated , for every quantile level $\alpha$, by
$$
     v(s_t,\alpha)\leftarrow v(s_t,\alpha)+\eta \mathbb{E}[\partial l^{\kappa}_{\epsilon,\alpha}(\delta_v(s_t,\alpha,r_t,s_{t+1},\tilde{u}))].
$$
When replacing the soft quantile loss $l^\kappa_{\epsilon,\alpha}(\cdot)$ by the original $l_\alpha(\cdot)$ (or huber one), the value update is the same as quantile regression in distributional RL~\cite{dabney2018distributional}. 





It is worth noting that although $\bar{\pi}^*$ attains the maximum of the objective in Eq.~\ref{eq:maximum-barpi} under the value function $v^{\bar{\pi}^*}$, it remains an open question whether iterative updates starting from arbitrary $\bar{\pi}$ and $v$ will converge to $\bar{\pi}^*$. Nevertheless, our empirical results indicate that incorporating this update rule with CVaR-PG substantially accelerates learning.

\subsection{Augmenting CVaR-PG with VaR-PG}
Before integrating the VaR-PG developed under the Markovian policy assumption with CVaR-PG to address the objective in Eq.~\ref{eq:obj}, several practical considerations remain to be addressed in the algorithm design.

\paragraph{Parameterization of quantile function $v$} Following quantile-based distributional RL~\cite{dabney2018distributional}, we discretize the quantile levels uniformly over $[0,1]$. The deep $v$-network takes the state as input and outputs $I$ quantile values, corresponding to the quantile levels $\{\frac{1}{2} (\frac{i-1}{I}+\frac{i}{I})\}_{i=1}^I$. However, this architecture is known to suffer from the quantile crossing issue, where the predicted quantile values can violate the required monotonicity in the quantile level, potentially impairing policy learning~\cite{zhou2020non,luo2022distributional}. To enforce monotonicity in a simple and effective manner, we use the following mechanism. The first network output is treated as a base value, and the remaining $I-1$ outputs are interpreted as incremental deltas, i.e.,
$$
\begin{aligned}
v(s)_1 &= f_v(s)_1\\
v(s)_i &=v(s)_1 + \sum_{j=2}^i \mathrm{softplus}(f_v(s)_j) ~~2\leq i\leq I,
\end{aligned}
$$
where $f_v$ denotes the deep value network, and the softplus function is applied to make deltas non-negative.

\paragraph{Multi-step advantage estimation}  Eq.~\ref{eq:mkv-adv} is one-step advantage estimation. Since we are now working with Markovian policy, the advantage can be estimated using multi-step rollout~\cite{schulman2016high} as
$$
\begin{aligned}
&\bar{A}^{(\iota)}(s_t,\alpha_t,a_t)=\mathbb{E}[\partial l^\kappa_{\epsilon,\alpha_t}(\delta_{v}^{(\iota)})]\\
\delta_{v}^{(\iota)}=-&v(s_t,\alpha_t)+r_t+\gamma r_{t+1}+...+\gamma^\iota v(s_{t+\iota},\tilde{u}).
\end{aligned}
$$
The weighted average forms the final advantage estimation
\begin{equation}
\label{eq:mkv-adv-multistep}
 \bar{A}(s_t,\alpha_t,a_t) = (1-\lambda)(\bar{A}^{(1)} +\lambda \bar{A}^{(2)}+\lambda^2 \bar{A}^{(3)} +...),   
\end{equation}
with $\lambda\in(0,1)$. 

It is now convenient to form the final algorithm to solve the objective in Eq.~\ref{eq:obj}. Given a Markovian policy $\bar\pi$ and its associated quantile value function $v$, at each iteration, $N$ trajectories are sampled by executing $\bar{\pi}$ in the environment. For each trajectory, the initial quantile level $\alpha_0$ is drawn from the uniform distribution $U[0,\alpha]$, and the subsequent quantile levels $\{\alpha_t\}$ are determined according to Eq.~\ref{eq:track-alpha} using $v$ and are recorded along the trajectory. The policy gradient for $\bar\pi$ is computed separately for the CVaR component (cf. Eq.~\ref{eq:cvar-pg}) and the VaR component (cf. Sec.~\ref{sec:mkv-policy} plus the averaged multi-step advantage), and the two gradients are combined via the trade-off parameter $\omega$. The value function $v$ is updated through quantile regression by choosing the quantile regression loss $l_\alpha(\cdot)$. The complete procedure is summarized in Algo.~\ref{alg:full} in Appendix. 

\textbf{Remark}. 
\textcolor{black}{Although our algorithm developed for VaR is amendable to focus on a single VaR objective, we instead work on CVaR, i.e., expected VaR, since}
VaR characterizes the return distribution only at a single quantile level and, consequently, optimizing \textcolor{black}{a single} VaR does not necessarily yield a sufficiently risk-averse policy, because outcomes in the tail beyond the specified quantile are ignored. Therefore, VaR optimization may fail to adequately penalize catastrophic events, thereby limiting its effectiveness in controlling extreme downside risk. \textcolor{black}{In addition, in our algorithm, the expected VaR term serves as a reinforcing (boosting) signal for CVaR-PG, and may eventually vanish by decaying $\omega$. This ensures that the convergence properties of our algorithm are the same as CVaR-PG.}
\section{Experiments}

\begin{figure}
    \begin{center}
        \includegraphics[width=0.95\columnwidth]{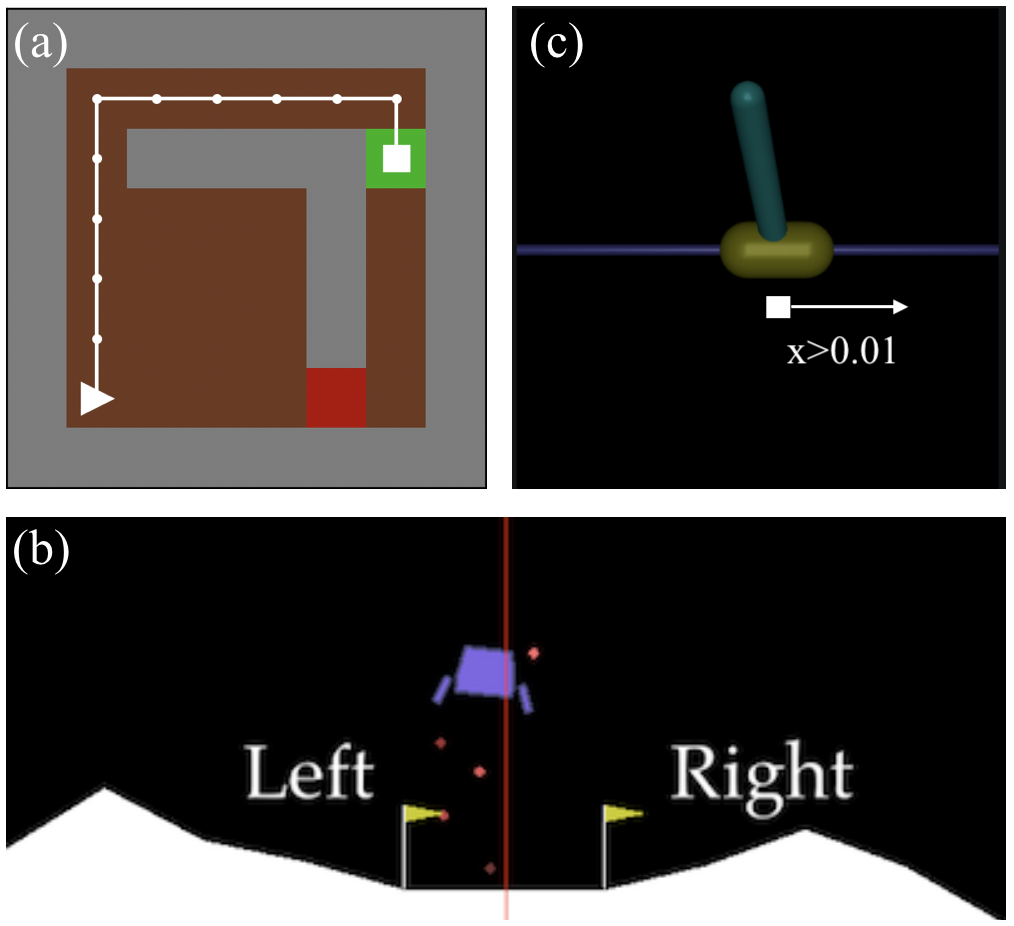}
    \end{center}
    \caption{(a) Maze. Visiting red state will receive a random reward, with mean $-1$. (b) LunarLander. Landing on the right part of the ground will receive a random reward with mean 0. (c) Inverted Pendulum. Staying in the region where $x>0.01$ will receive a random reward with mean 0 per step.}
    \label{fig:envs}
\end{figure}

Following \citet{luo2024simple}, we modify several domains such that the risk-averse policy is clear to identify to evaluate the algorithms. We start from a simple maze domain adapted from \citet{greenberg2022efficient}, then LunarLander from the Box2D environments of OpenAI Gym~\cite{brockman2016openai} and InvertedPendulum from the Mujoco~\cite{todorov2012mujoco} environments of OpenAI Gym (We did not include some other MuJoCo domains considered in \citet{luo2024simple}, such as Ant and HalfCheetah. In these environments, agents are free to move both forward and backward, and we observe that policies may achieve high returns while  wandering near the origin, which obscures meaningful notions of risk-averse behavior. The design of environments that admit well-defined and practically relevant risk-averse objectives is itself a nontrivial research problem and is beyond the scope of this paper).


\textbf{Baselines}. We compare our method with the original CVaR-PG (cf. Eq.~\ref{eq:cvar-pg})~\cite{tamar2015optimizing}, Predictive CVaR-PG (PCVaR-PG)~\cite{kim2024risk}, Return Capping (RET-CAP)~\cite{mead2025return}, and the mixture policy method (denoted as MIX) in \citet{luo2024simple}. The method in \citet{greenberg2022efficient} requires the control of environment randomness, and the method in \citet{lim2022distributional} is shown to not work very well in \citet{luo2024simple}, so these two are not included. We also include REINFORCE with baseline method~\cite{sutton1998reinforcement} as a risk-neutral baseline. We denote our method as CVaR-VaR in what follows, and recall that CVaR-VaR reduces to CVaR-PG when $\omega=0$.

\textbf{Settings}. To ensure a relatively fair comparison, we implement all methods using on-policy PG without off-policy importance sampling (IS), thereby eliminating the potential confounding effects of IS ratio clip. The only exception is MIX, whose risk-neutral component is updated by off-policy data via offline RL (IQL~\cite{kostrikov2022offline}). Across all domains, methods are implemented with deep neural networks, and the policy is updated after collecting $N=20$ trajectories. Consequently, the results may differ slightly from those reported in \citet{mead2025return} and \citet{luo2024simple}: the former employs off-policy learning with IS ratio clip, while the latter collects a larger number of trajectories for policy updates. Additional implementation and parameter details are provided in Appendix~\ref{sec:exp}.

\begin{figure}[t]
    \begin{center}
        \includegraphics[width=0.99\columnwidth]{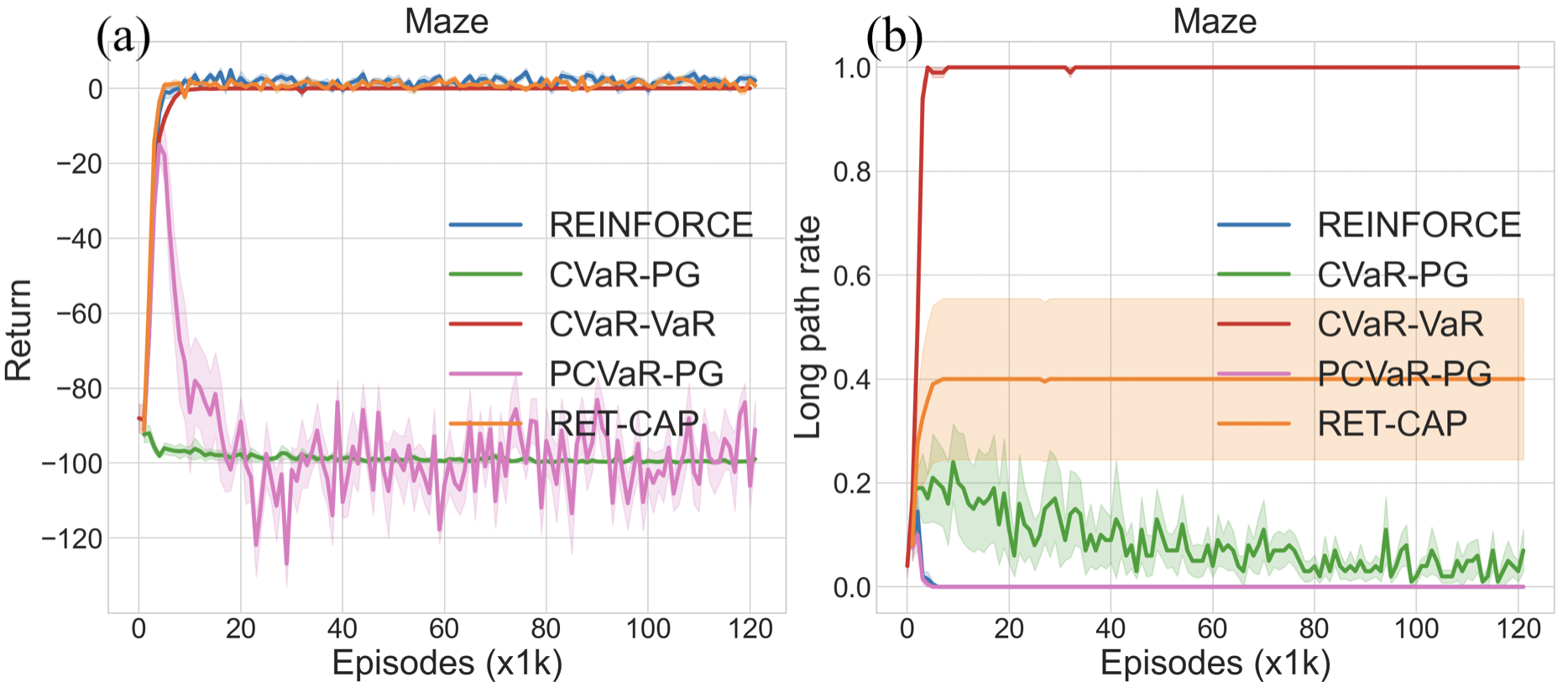}
    \end{center}
    \caption{(a) Expected return (b) Long (risk-averse) path rate in Maze. Curves are averaged over 10 seeds with shaded regions indicating standard errors.}
    \label{fig:maze-res}
\end{figure}

\subsection{Maze}
The maze domain is depicted in Fig.~\ref{fig:envs}a.  The gray color marks the walls. The agent starts from the white triangle state and aims to reach the green goal state. The action space is discrete with four actions $\{$Up, Down, Left, Right$\}$. The per-step reward is $-1$ before reaching the goal except the red state. Visiting the red state receives a random reward $-1 + \mathcal{N}(0,1)\times 30$. The reward for visiting the goal is $10$. Thus, the shortest path going through the red state towards the goal is the optimal risk-neutral path, while the longer path (marked in white color) is $\alpha$-CVaR optimal with small $\alpha$. The maximum episode length is 100. We set CVaR $\alpha=0.1$. Note that MIX fixes its risk-neutral component to the optimal risk-neutral policy in this domain, and is therefore not included in the comparison here. RET-CAP requires the $\alpha$-VaR of the return distribution induced by the optimal $\alpha$-CVaR policy, which is available in this domain.

We report the return and long path rate of different methods in Fig.~\ref{fig:maze-res}. CVaR-PG fails to learn a reasonable policy in this domain since the return distribution exhibits a flat left tail during exploring. This leads to vanishing gradient as discussed in \citet{greenberg2022efficient}. PCVaR-PG initially acquires reward but ultimately fails, since it relies on learning a prediction function which takes cumulative rewards as input, this component may be difficult to approximate reliably with function approximation. Both RET-CAP and CVaR-VaR are able to learn policies with high returns but RET-CAP does not consistently converge to the risk-averse policy across all random seeds. In contrast, CVaR-VaR rapidly converges to a stable risk-averse policy.

\begin{figure*}[t]
    \centering



    \begin{subfigure}[t]{0.28\textwidth}
        \centering
        \includegraphics[width=\linewidth]{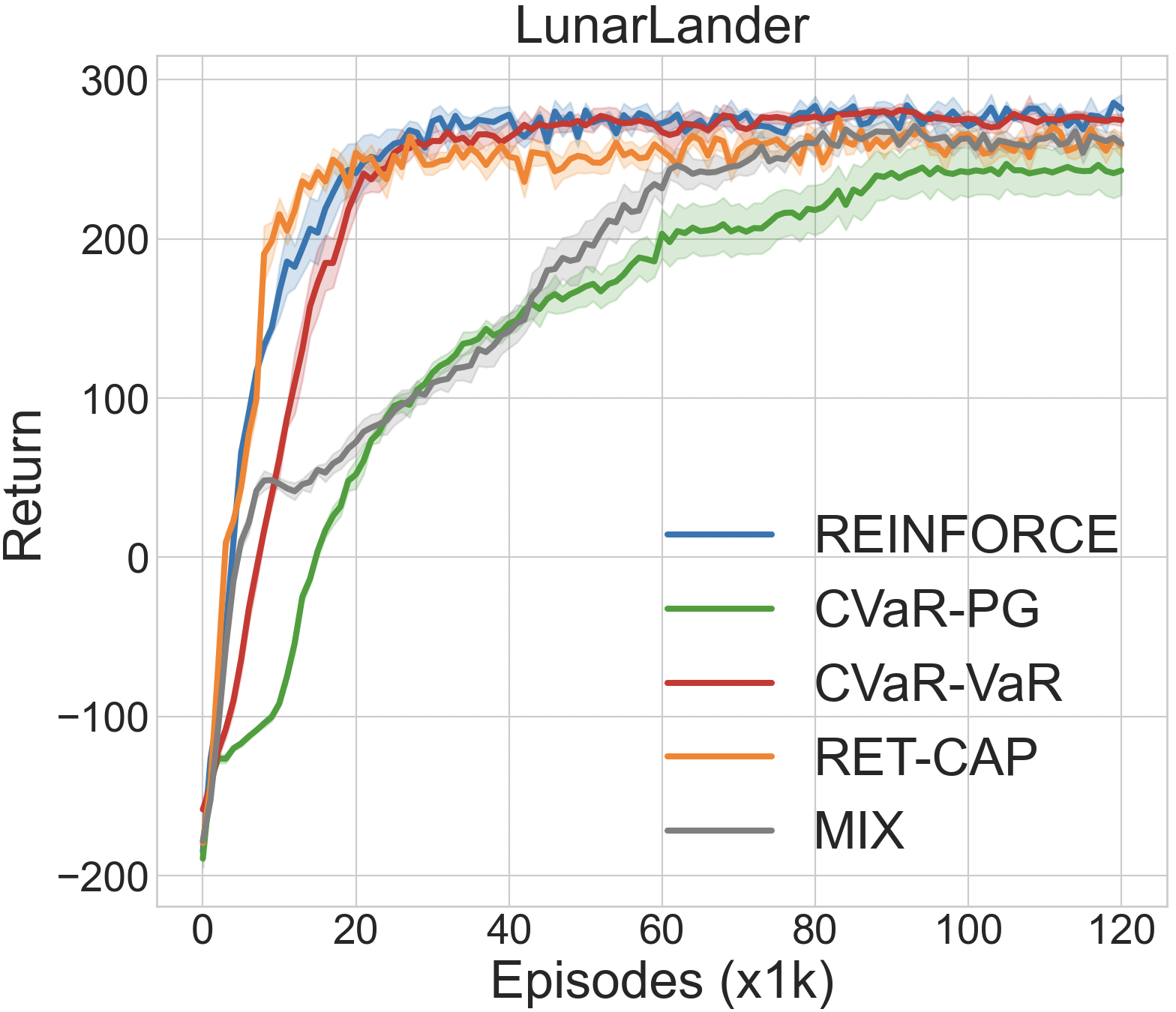}
        \caption{Expected return in Lunarlander}
        \label{fig:sub3}
    \end{subfigure}
    \hspace{0.02\textwidth}
    \begin{subfigure}[t]{0.28\textwidth}
        \centering
        \includegraphics[width=\linewidth]{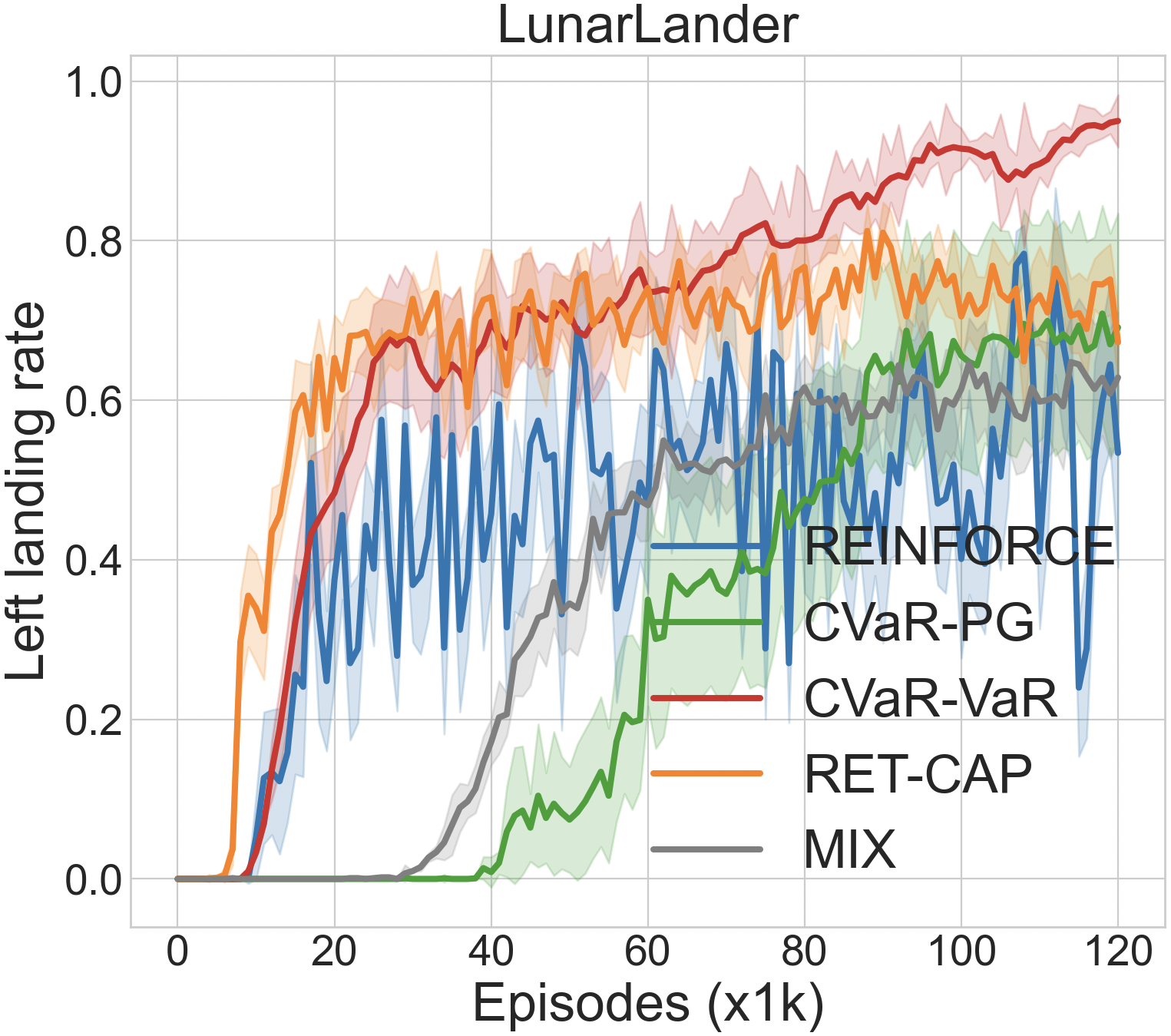}
        \caption{Risk-averse rate in Lunarlander}
        \label{fig:sub4}
    \end{subfigure}
    \hspace{0.02\textwidth}
    \begin{subfigure}[t]{0.28\textwidth}
        \centering
        \includegraphics[width=\linewidth]{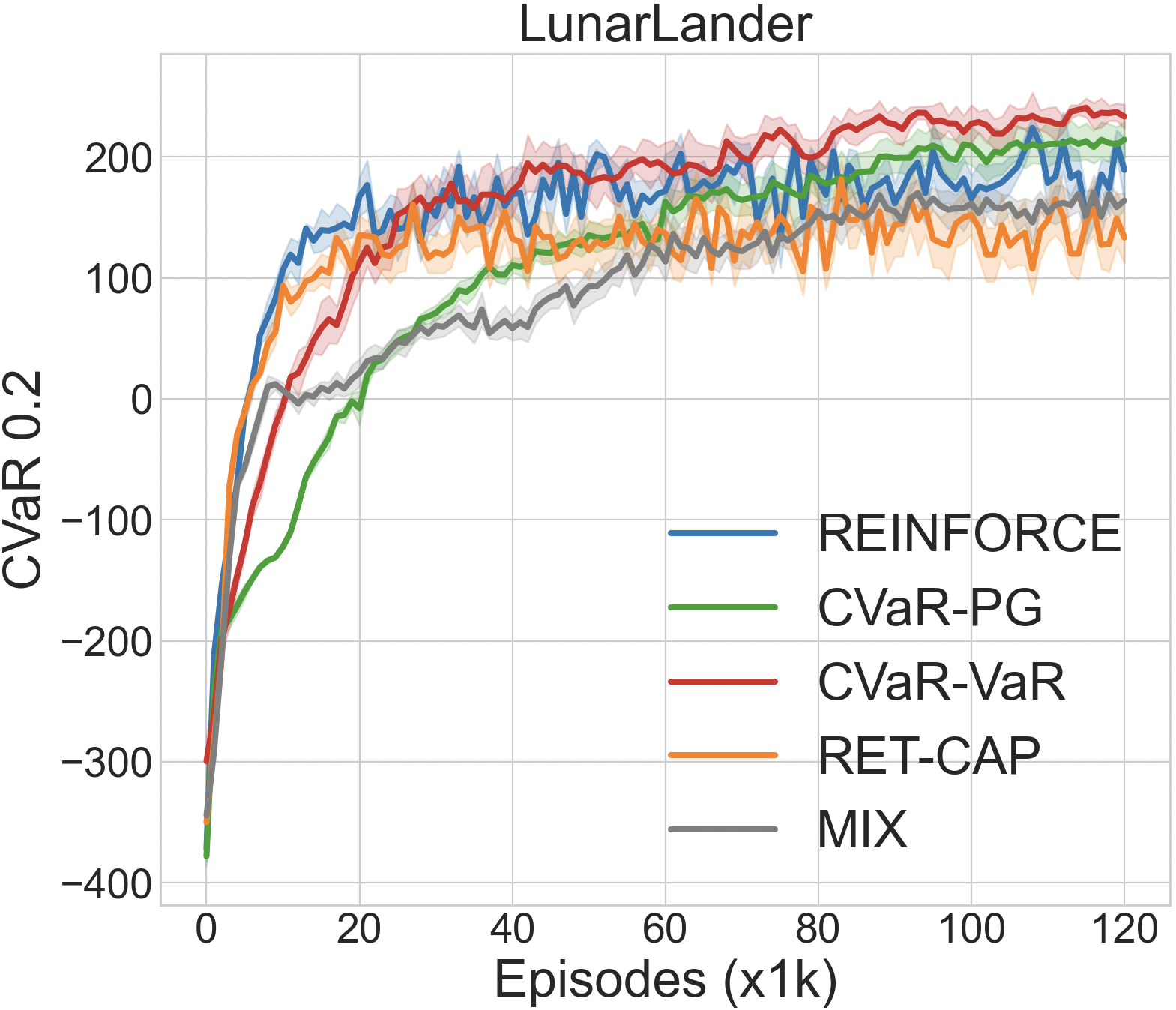}
        \caption{CVaR 0.2 of return in Lunarlander}
        \label{fig:sub5}
    \end{subfigure}

    \begin{subfigure}[t]{0.28\textwidth}
        \centering
        \includegraphics[width=\linewidth]{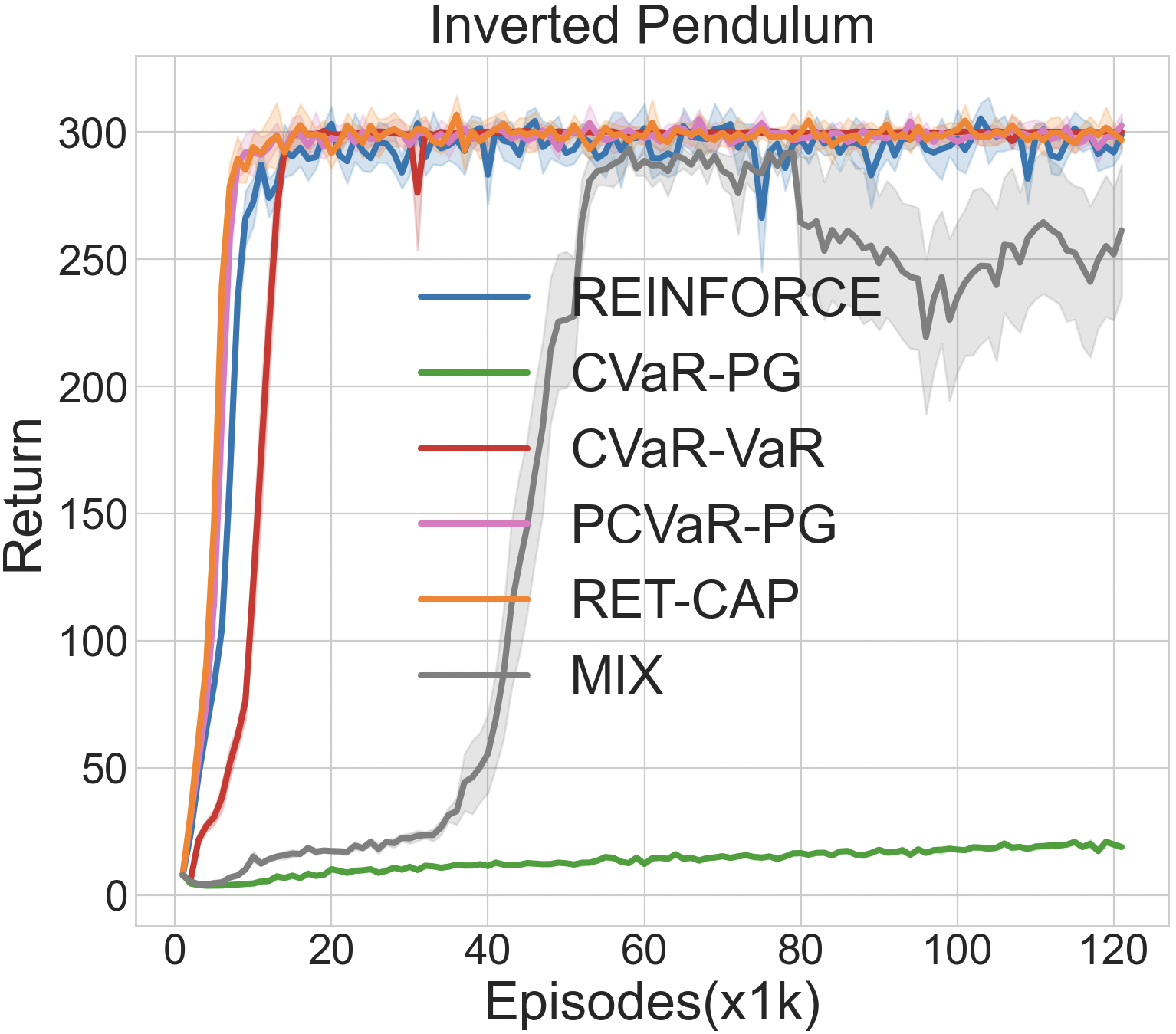}
        \caption{Expected return in Pendulum}
        \label{fig:sub6}
    \end{subfigure}
    \hspace{0.02\textwidth}
    \begin{subfigure}[t]{0.28\textwidth}
        \centering
        \includegraphics[width=\linewidth]{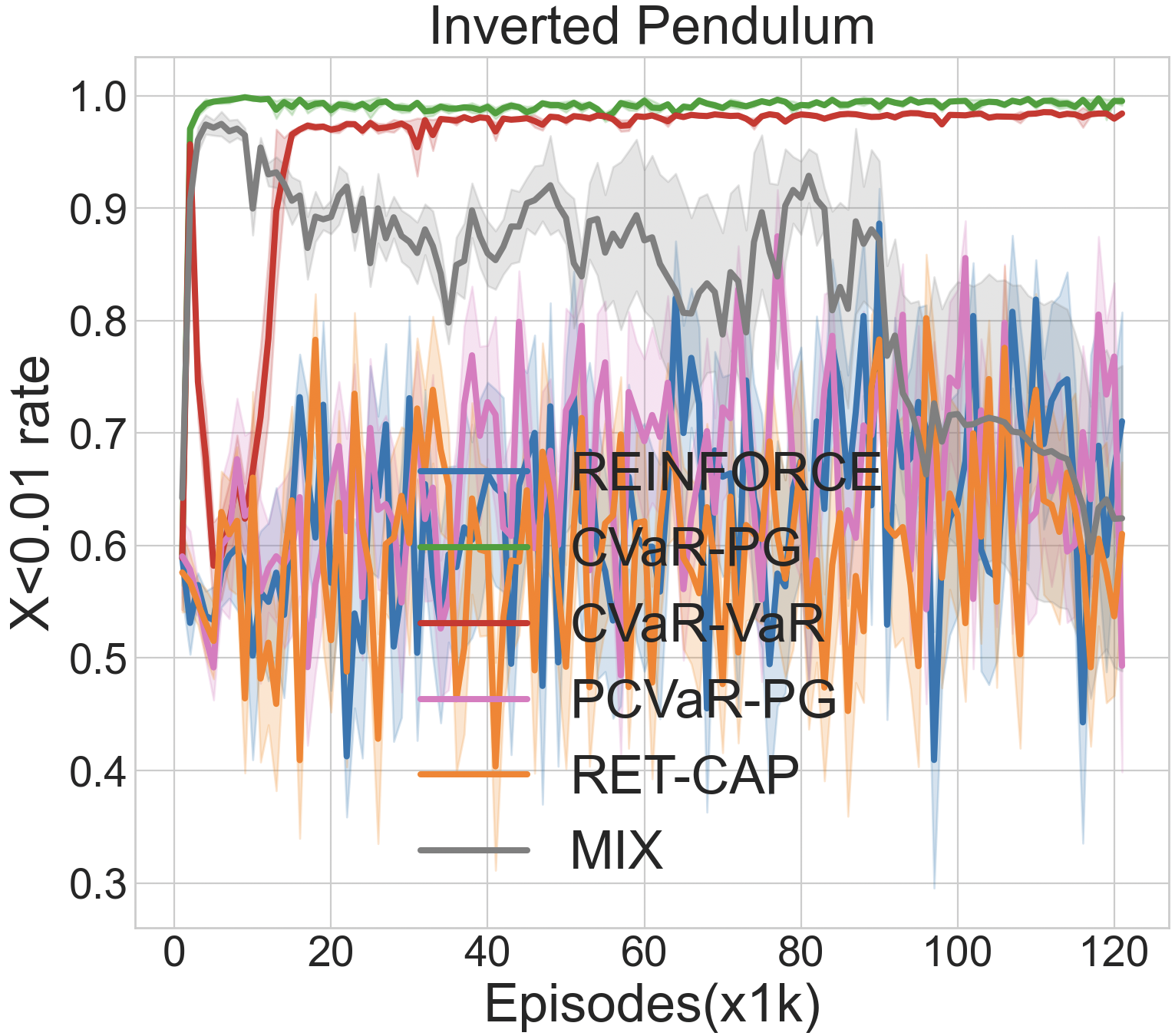}
        \caption{Risk-averse rate in Pendulum}
        \label{fig:sub7}
    \end{subfigure}
    \hspace{0.02\textwidth}
    \begin{subfigure}[t]{0.28\textwidth}
        \centering
        \includegraphics[width=\linewidth]{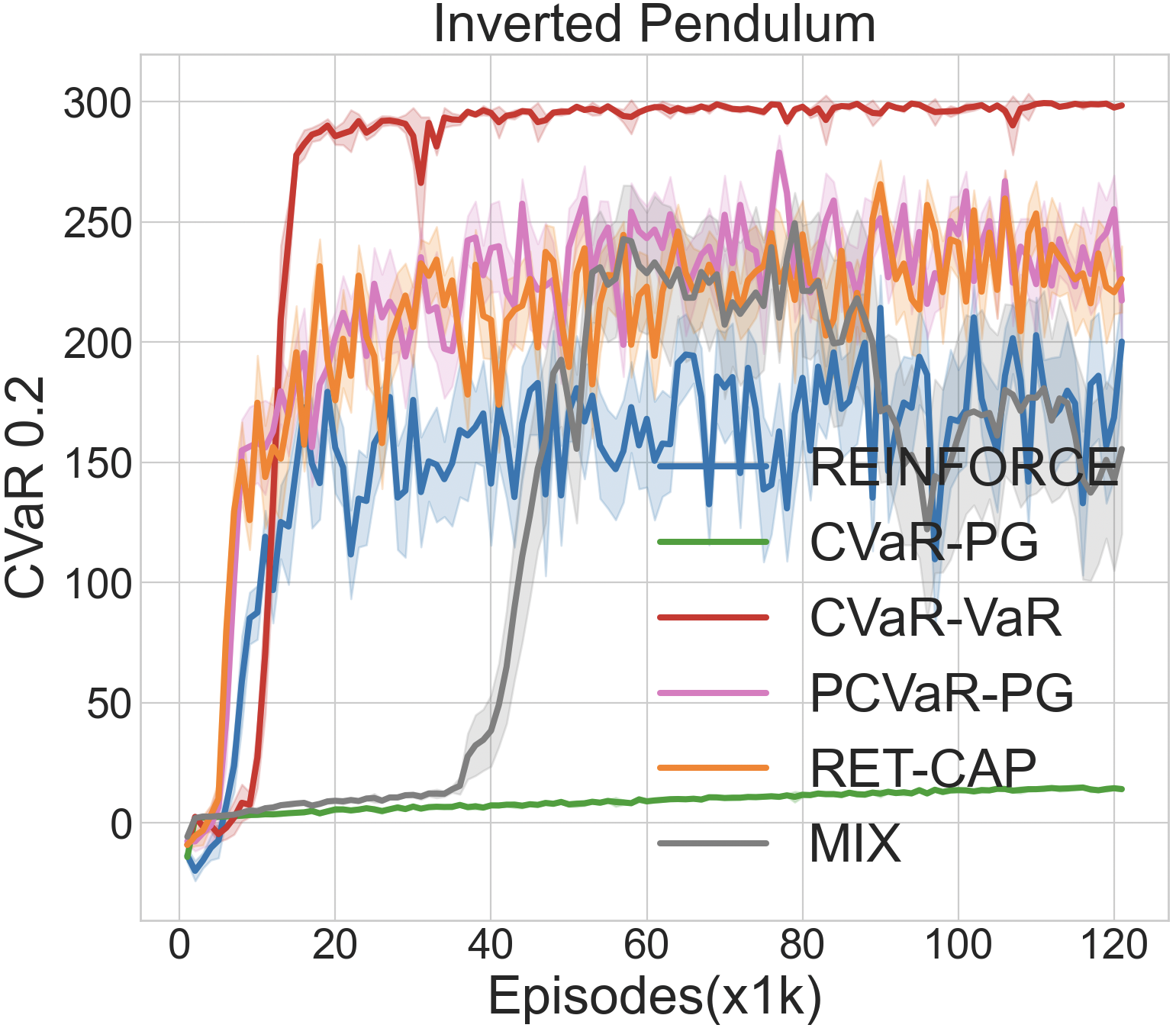}
        \caption{CVaR 0.2 of return in Pendulum}
        \label{fig:sub8}
    \end{subfigure}

    \caption{Expected return, risk-averse rate, and CVaR 0.2 of return in LunarLander and Inverted Pendulum. Curves are averaged over 10 seeds with shaded regions indicating standard errors.}
    \label{fig:ll-ivp-res}
    \vspace{-0.2in}
\end{figure*}

\subsection{LunarLander}
This domain is shown in Fig.~\ref{fig:envs}b. The goal in this domain is to control the engine of the lander to safely land on the ground without crashing. We refer readers to its official
documents for a full description. Originally, successful landing receives a $100$ reward. Here, we split the ground into left and right parts by the middle line of the landing pad. If landing on the right region, an additional noisy reward $\mathcal{N}(0,1)\times 100$ is given. Therefore, a risk-averse agent should learn to land on the left as much as possible. The maximum episode length is 500. We set CVaR $\alpha=0.2$.

The return and left landing rate are shown in Fig.~\ref{fig:sub3} and \ref{fig:sub4}. We also report the $0.2$-CVaR of return in Fig.~\ref{fig:sub5}. PCVaR-PG fails to learn a reasonable policy in this domain and is therefore omitted from the plots. Compared with CVaR-PG, both RET-CAP and CVaR-VaR quickly achieve high returns, with CVaR-VaR attaining slightly higher performance than RET-CAP. MIX provides only marginal improvement over CVaR-PG, consistent with the results reported in~\citet{luo2024simple}. In terms of risk-aversion rate, only CVaR-VaR approaches a left-landing rate close to one in the end. Consequently, CVaR-VaR also attains the highest $0.2$-CVaR of return among all methods.

\subsection{InvertedPendulum}
This domain is shown in Fig.~\ref{fig:envs}c. The goal is to keep the inverted pendulum upright (within a certain angle limit) as long as possible. The original per-step reward is 1. We add a noise $\mathcal{N}(0,1)\times10$ to the pre-step reward if X-position $>0.01$. Thus, a risk-averse agent should learn to balance the pendulum while stay away from the noisy reward region. The maximum episode length is 300. We set CVaR $\alpha=0.2$.

The return, X$<0.01$ rate, and $0.2$-CVaR of return are presented in Fig.~\ref{fig:sub6}, \ref{fig:sub7}, and \ref{fig:sub8}. CVaR-PG learns relatively slowly in this domain compared with others methods, although it demonstrates strong risk-averse behavior (a high X$<0.01$ rate). PCVaR-PG and RET-CAP achieve high returns, however, they fail to consistently avoid the noisy reward region. MIX accelerates learning relative to CVaR-PG but remains slower than the remaining approaches; as a hybrid of online and offline learning method, its performance is sensitive to the choice of hyperparameters. By contrast, CVaR-VaR achieves both high return and strong risk aversion, so it also reflects a high $0.2$-CVaR of return.
\section{Conclusion and Future Work}
This paper proposes augmenting CVaR optimization with VaR to improve sample efficiency. It is motivated by recent advances in dynamic programming methods for VaR optimization, as well as the observation that CVaR corresponds to the expectation of VaR over the tail of the return distribution. We first develop a VaR Bellman operator that yields an actor–critic algorithm for VaR, and then adapt it to operate within the Markovian policy class. Empirically, We show that our method can succeed when others fail to learn a risk-averse or a reasonable policy by mitigating the sample efficiency issue.

\textbf{Limitations and future work}. As discussed, the objective in Eq.~\ref{eq:maximum-barpi} depends on the value function associated with the optimal Markovian VaR policy, and its convergence properties when initialized from an arbitrary value function remain unknown. In addition, our method could potentially be integrated with other techniques aimed at enhancing sample efficiency e.g.,~\citet{greenberg2022efficient,mead2025return}. Observing the two limitations, analyzing the convergence behavior of the proposed Markovian VaR algorithm, as well as possible integration with existing methods to improve sample efficiency, remain valuable for future work.

\section*{Impact Statement}
This paper presents work whose goal is to advance the field of risk-averse reinforcement learning. There are many potential societal consequences of our work, none which we feel must be specifically highlighted here.

\section*{Acknowledgements}

Y. Luo was partially funded by FRQNT, and IVADO.
E. Delage was partially supported by the Canadian Natural Sciences and Engineering
Research Council [Grant RGPIN-2022-05261] and by the Canada Research Chair program [950-
230057].Finally, this research was enabled in part by support provided by the Digital Research Alliance of Canada (\url{https://www.alliancecan.ca/}).

\nocite{langley00}

\bibliography{reference}
\bibliographystyle{icml2026}

\newpage
\appendix
\onecolumn
\section{Proofs.}
We consider the soft quantile loss $l^\kappa_\alpha(\cdot)$ defined as in \citet{hau2025q} for proofs. We also clip the $\alpha$ such that it lies in range $[\epsilon,1-\epsilon]$ with small $\epsilon$, i.e., $l^\kappa_{\max(\epsilon,\min(1-\epsilon,\alpha))}(\cdot)$, denoted as $l^\kappa_{\epsilon,\alpha}(\cdot)$. When $\epsilon= 0$ and $\kappa=0$, it becomes the quantile regression loss $l_\alpha(\cdot)$. We also need the derivative of $l_\alpha^\kappa(\cdot)$, denoted as $\partial l_\alpha^\kappa(\cdot)$.
\begin{equation}
\begin{alignedat}{2}
    l_\alpha^\kappa(\delta)=
\left\{
\begin{alignedat}{2}
 &\frac{(1-\alpha)\kappa}{2}((\delta+\kappa)^2-\frac{2\delta}{\kappa}-1) &&~\text{if}~\delta<-\kappa \\
&(1-\alpha)\frac{\delta^2}{2\kappa} &&~\text{if}~\delta\in[-\kappa,0)\\
& \alpha \frac{\delta^2}{2\kappa} &&~\text{if}~\delta\in[0,\kappa)\\
& \frac{\alpha\kappa}{2}((\delta-\kappa)^2 + \frac{2\delta}{\kappa}-1) &&~\text{if}~\delta\geq\kappa.
\end{alignedat}
\right.
&\qquad
\begin{alignedat}{2}
    \partial l_\alpha^\kappa(\delta)=
\left\{
\begin{alignedat}{2}
 &(1-\alpha)(\kappa \delta+\kappa^2-1) &&~\text{if}~\delta<-\kappa \\
&\frac{1-\alpha}{\kappa} \delta &&~\text{if}~\delta\in[-\kappa,0)\\
& \frac{\alpha}{\kappa} \delta &&~\text{if}~\delta\in[0,\kappa)\\
& \alpha(\kappa \delta - \kappa^2 + 1)&&~\text{if}~\delta\geq\kappa.
\end{alignedat}
\right.
\end{alignedat}
\end{alignedat}
\end{equation}
From the definition of $\partial l^\kappa_\alpha(\cdot)$, we know the derivative of $\partial l^\kappa_\alpha(\cdot)$ lies in the range $[\min\{\alpha,1-\alpha\}\kappa, \max\{\alpha,1-\alpha\}\kappa^{-1}]$. Therefore, the derivative of $\partial l^\kappa_{\epsilon,\alpha}(\cdot)$ lies in the range $[\epsilon\kappa, (1-\epsilon)\kappa]$. 


\subsection{Fixed point of nested VaR Bellman optimality equation}\label{app:VaRinfHor}

\begin{proposition}\label{thm:BfixedPoint}
The Bellman optimality equation
\begin{equation}
\label{eq:qdp-nested2}
    q(s,\alpha,a)=\mathrm{VaR}_\alpha[\tilde{r}(s,a) + \gamma \max_{a'} q(\tilde{s}',\tilde{u},a')]
    \vspace{-0.05in}
\end{equation}
    has a unique fixed point at $q^*$.
\end{proposition}

\begin{proof}
    We first demonstrate that the Bellman operator
    \[\mathcal{B}q(s,\alpha,a) := \mathrm{VaR}_\alpha[\tilde{r}(s,a) + \gamma \max_{a'} q(\tilde{s}',\tilde{u},a')]\]
    is a $\gamma$-contraction according to the sup-norm and therefore admits a unique fixed point that we denote as $\hat{q}$. Namely, take any two function $q^1$ and $q^2$. One can verify:
    \begin{align*}
        \mathcal{B}&q^1(s,\alpha,a) - \mathcal{B}q^2(s,\alpha,a) \\
        &= \mathrm{VaR}_\alpha[\tilde{r}(s,a) + \gamma \max_{a'} q^1(\tilde{s}',\tilde{u},a')] - \mathrm{VaR}_\alpha[\tilde{r}(s,a) + \gamma \max_{a'} q^2(\tilde{s}',\tilde{u},a')]\\
        &\leq \mathrm{VaR}_\alpha[ \tilde{r}(s,a) + \gamma\max_{a'}(q^2(\tilde{s}',\tilde{u},a')+\|q_1-q_2\|_\infty)]\\
        &\qquad\qquad- \mathrm{VaR}_\alpha[ \tilde{r}(s,a) + \gamma\max_{a'}q^2(\tilde{s}',\tilde{u},a')]\\
        &=\gamma  \|q_1-q_2\|_\infty.
    \end{align*}    

    Next, for any $T>0$, the finite horizon value function:
    \[q_T^{*}(s,\alpha,a):= \max_{\pi\in\Pi_{\mathcal{HD}}}\mathrm{VaR}^\pi_\alpha[\sum_{t=0}^{T-1}\tilde{r}(\tilde{s}_t,\tilde{a}_t)|\tilde{s}_0=s,\tilde{a}_0=a]\]
    is known, based on Theorem 3.2 in \citet{hau2025q}, to satisfy:
    \[q_T^{*}(s,\alpha,a) = \mathcal{B}^T q_0(s,\alpha,a)\]
    with $q_0(s,\alpha,a):=0$. The $\gamma$-contraction property of $\mathcal{B}$ implies that:
    \[\|q_T^{*}-\hat{q}\|_\infty = \|\mathcal{B}^T q_0(s,\alpha,a)-\mathcal{B}^T\hat{q}\|_\infty \leq \gamma^T\|q_0(s,\alpha,a)-\hat{q}\|_\infty = \gamma^T\|\hat{q}\|_\infty.\]
    This confirms that  $\hat{q}(s,\alpha,a)= \lim_{T \rightarrow \infty}q_T^*(s,\alpha,a)=q^*(s,\alpha,a)$.
\end{proof}

\subsection{VaR Bellman Optimality Equations with random reward}
\label{sec:eq-random-reward}

\newcommand{\tildey}{{\tilde{y}}}
\newcommand{\1}{{\mathbf{1}}}
We rederive the VaR Bellman equation presented in  \citet{hau2025q} for the case of finite horizon, and finite state and action space but now with finitely supported random rewards  (rather than deterministic). Namely, let's define:
\[
q_t^*(s,\alpha,a):=\max_{\pi\in\Pi^t}\mathrm{VaR}^\pi_\alpha\Big[\sum_{k=0}^{t-1} \gamma^{k} \tilde{r}(\tilde{s}_{k},\tilde{a}_k)|\tilde{s}_0=s,\tilde{a}_0=a\Big],\]
where $\Pi^t$ is the set of all history dependent policies with a history at time $k$ take the form $(s_0,a_0,r_0,s_1,\dots,s_k)$.  Since the set of deterministic policy is enumerable, an optimal history dependent policy $\pi^*$ must exists and we must have that:
\[q_t^*(s,\alpha,a)=\mathrm{VaR}^{\pi^*}_\alpha\Big[\sum_{k=0}^{t-1} \gamma^{k} \tilde{r}(\tilde{s}_{k},\tilde{a}_k)|\tilde{s}_0=s,\tilde{a}_0=a\Big],\]
thus ensuring that $q_t^*(s,\alpha,a)$ is non-decreasing and right-continuous in $\alpha$.

We also recall a VaR decomposition lemma from \citet{hau2025q}.
\begin{lemma}{[Lemma B.4 of \citet{hau2025q}]}
\label{lemma:buhuiqiming}
    Let $n$ non-decreasing functions $f_i:[0,1]\rightarrow \bar{\mathbb{R}}$ ($\bar{\mathbb{R}}:=\mathbb{R}\cup \{-\infty,+\infty\}$). Let $\tilde{y}$ be a discrete random variable on $1:n$ with probability mass function $p_i=\mathbb{P}[\tilde{y}=i]$, and $\tilde{u}$ an independent random variable with uniform distribution on $[0,1]$, then
    $$
    \mathrm{VaR}_\alpha[f_{\tilde{y}}(\tilde{u})]=\max_{\boldsymbol{\beta}\in\mathcal{B}(\alpha)}\Big\{\min_i \mathrm{VaR}_{\beta_i}[f_i(\tilde{u})]\Big\},
    $$
    where $\mathcal{B}(\alpha):=\{\boldsymbol{\beta}\in [0,1]^n |\sum_{i=1}^n \beta_i p_i \leq \alpha \}$.
\end{lemma}

Letting $\{(r_i,s_i)\}_{i=1}^n$ capture the space of possible distinct realizations of $(\tilde{r}(s,a),s')$ and 
$\tildey$ be a discrete random variable in $\{1,\dots,n\}$ with probability $p_i=P(\tildey=i|s,a)=P((\tilde{r}(s,a),\tilde{s}')=P((s_i,r_i)|s,a)=P(s_i|s,a)P(\tilde{r}(s,a)=r_i|s,a,s_i)$, 
Then, we have:
\begin{align*}
q_t^*(s,\alpha,a)&:=\max_{\pi\in\Pi^t}\mathrm{VaR}^\pi_\alpha\Big[\sum_{k=0}^{t-1} \gamma^{k} \tilde{r}(\tilde{s}_{k},\tilde{a}_k)|\tilde{s}_0=s,\tilde{a}_0=a\Big]\\
&\overset{(a)}{=}\max_{\pi\in\Pi^t}\mathrm{VaR}^\pi_\alpha\Big[\mathrm{VaR}^\pi_{\tilde{u}}[\sum_{k=0}^{t-1} \gamma^{k} \tilde{r}(\tilde{s}_{k},\tilde{a}_k)|\tilde{s}_0=s,\tilde{a}_0=a,\tilde{s}_1=s_
\tildey,\tilde{r}(s,a)=r_\tildey]\Big]\\
&\overset{(b)}{=}\max_{\pi\in\Pi^t}\max_{\boldsymbol{\beta}\in\mathcal{B}} \min_i \mathrm{VaR}_{\beta_i}[\mathrm{VaR}^\pi_{\tilde{u}}[\sum_{k=0}^{t-1} \gamma^{k} \tilde{r}(\tilde{s}_{k},\tilde{a}_k)|\tilde{s}_0=s,\tilde{a}_0=a,\tilde{s}_1=s_i,\tilde{r}(s,a)=r_i]]\\
&\overset{(c)}{=}\max_{\boldsymbol{\beta}\in\mathcal{B}} \max_{\pi\in\Pi^t}\min_i r_i + \gamma \mathrm{VaR}_{\beta_i}[\mathrm{VaR}^\pi_{\tilde{u}}[\sum_{k=1}^{t-1} \gamma^{k-1} \tilde{r}(\tilde{s}_{k},\tilde{a}_k)|\tilde{s}_0=s,\tilde{a}_0=a,\tilde{s}_1=s_i,\tilde{r}(s,a)=r_i]]\\
&\overset{(d)}{=}\max_{\boldsymbol{\beta}\in\mathcal{B}} \min_i \max_{\pi\in\Pi^{t}}r_i + \gamma \mathrm{VaR}_{\beta_i}[\mathrm{VaR}^{\pi}_{\tilde{u}}[\sum_{k=1}^{t-1} \gamma^{k-1} \tilde{r}(\tilde{s}_{k},\tilde{a}_k)|\tilde{s}_0=s,\tilde{a}_0=a,\tilde{s}_1=s_i,\tilde{r}(s,a)=r_i]]\\
&\overset{(e)}{=}\max_{\boldsymbol{\beta}\in\mathcal{B}} \min_i \max_{\pi\in\Pi^{t-1}}r_i + \gamma \mathrm{VaR}_{\beta_i}[\mathrm{VaR}^{\pi}_{\tilde{u}}[\sum_{k=0}^{t-2} \gamma^{k} \tilde{r}(\tilde{s}_{k},\tilde{a}_k)|\tilde{s}_0=s_i]]\\
&\overset{(f)}{=}\max_{\boldsymbol{\beta}\in\mathcal{B}} \min_i \max_a \max_{\pi\in\Pi^{t-1}}r_i + \gamma \mathrm{VaR}_{\beta_i}[\mathrm{VaR}^{\pi}_{\tilde{u}}[\sum_{k=0}^{t-2} \gamma^{k} \tilde{r}(\tilde{s}_{k},\tilde{a}_k)|\tilde{s}_0=s_i,a_0=a]]\\
&\overset{(g)}{=}\max_{\boldsymbol{\beta}\in\mathcal{B}} \min_i \max_a \max_{\pi\in\Pi^{t-1}}r_i + \gamma \mathrm{VaR}_{\beta_i}[\sum_{k=0}^{t-2} \gamma^{k} \tilde{r}(\tilde{s}_{k},\tilde{a}_k)|\tilde{s}_0=s_i,a_0=a]\\
&\overset{(h)}{=}\max_{\boldsymbol{\beta}\in\mathcal{B}} \min_i \max_a r_i + \gamma q_{t-1}^*(s_i,\beta_i,a)\\
&\overset{(i)}{=}\max_{\boldsymbol{\beta}\in\mathcal{B}} \min_i \mathrm{VaR}_{\beta_i}[\max_a r_i + \gamma q_{t-1}^*(s_i,\tilde{u},a)]\\
&\overset{(j)}{=}\mathrm{VaR}_{\alpha}[\tilde{r}(s,a)+ \gamma \max_a q_{t-1}^*(\tilde{s}',\tilde{u},a)],
\end{align*}
where (a) follows from the fact that conditionally on $\tilde{s}_0=s,\tilde{a}_0=a$, we have 
\[\sum_{k=0}^{t-1} \gamma^{k} \tilde{r}(\tilde{s}_{k},\tilde{a}_k) = \mathrm{VaR}^\pi_{\tilde{u}}[\sum_{k=0}^{t-1} \gamma^{k} \tilde{r}(\tilde{s}_{k},\tilde{a}_k)|\tilde{s}_0=s,\tilde{a}_0=a,\tilde{s}_1=s_
\tildey,\tilde{r}(s,a)=r_\tildey]\]
in distribution. (b) follows from Lemma \ref{lemma:buhuiqiming}. (c) follows from the commutativity of maximum operators. (d) follows from the fact that all $\pi\in\Pi^t$ adapt to the realized $(r_i,s_i)$ for actions after actions $a_k$ with $k>0$. (e) follows from the fact that the function being optimized only depends on $\pi_{1:t-1}$ and that $(\tilde{r}_0,\tilde{s}_1)$ are already fixed, so that the trajectory can be reindexed using $k':=k-1$. (f) adds a redundant optimization of action $a_0$ in the reindexed trajectory. (g) replaces the distribution of $\mathrm{VaR}^{\pi}_{\tilde{u}}[\sum_{k=0}^{t-2} \gamma^{k} \tilde{r}(\tilde{s}_{k},\tilde{a}_k)|\tilde{s}_0=s_i,a_0=a]$ with the equivalent conditional distribution of $\sum_{k=0}^{t-2} \gamma^{k} \tilde{r}(\tilde{s}_{k},\tilde{a}_k$ given $\tilde{s}_0=s_i,a_0=a$. (h) exploits the definition of $q^*_{t-1}(s,\alpha,a)$. (i) exploits Lemma B.2 in \citet{hau2025q} and right-continuity of $\max_a r_i+\gamma q^*_{t-1}(s,\alpha,a)$ in $\alpha$. Finally, (j) reuses Lemma \ref{lemma:buhuiqiming}.

\subsection{Proof of Proposition~\ref{prop:contraction}}
\Propcontraction*
\begin{proof}
$$
\begin{aligned}
&\mathcal{T}_{\epsilon,\kappa}^*v_1(s,\alpha) - \mathcal{T}_{\epsilon,\kappa}^* v_2(s,\alpha)\\
=& v_1(s,\alpha) +\max_a \eta \mathbb{E}\Big[\partial l^\kappa_{\max(\epsilon,\min(1-\epsilon,\alpha))}(\tilde{r}(s,a)+\gamma v_1(\tilde{s}',\tilde{u})-v_1(s,\alpha))\Big] - v_2(s,\alpha)\\
&-\max_a \eta \mathbb{E}\Big[\partial l^\kappa_{\max(\epsilon,\min(1-\epsilon,\alpha))}(\tilde{r}(s,a)+\gamma v_2(\tilde{s}',\tilde{u})-v_2(s,\alpha))\Big]\\
\overset{(a)}{\leq}& v_1(s,\alpha)-v_2(s,\alpha) + \max_a \eta\mathbb{E}\Big[\partial l^\kappa_{\max(\epsilon,\min(1-\epsilon,\alpha))}(\tilde{r}(s,a)+\gamma v_1(\tilde{s}',\tilde{u})-v_1(s,\alpha))\\
&-\partial l^\kappa_{\max(\epsilon,\min(1-\epsilon,\alpha))}(\tilde{r}(s,a)+\gamma v_2(\tilde{s}',\tilde{u})-v_2(s,\alpha))\Big]\\
\overset{(b)}{=}&v_1(s,\alpha)-v_2(s,\alpha) + \max_a\eta\mathbb{E}\Big[\xi^\epsilon(s,a,\tilde{r},\tilde{s}',\tilde{u},v_1,v_2)(\gamma v_1(\tilde{s}',\tilde{u})-\gamma v_2(\tilde{s}',\tilde{u})-v_1(s,\alpha)+v_2(s,\alpha))\Big], \\
&\epsilon\kappa \leq \xi^\epsilon(s,a,\tilde{r},\tilde{s}',\tilde{u}, v_1, v_2) \leq (1-\epsilon)\kappa^{-1}\\
=&\max_a \mathbb{E}\Big[\big(1-\eta \xi^\epsilon(s,a,\tilde{r},\tilde{s}',\tilde{u},v_1,v_2)\big) (v_1(s,\alpha)-v_2(s,\alpha)) + \eta \xi^\epsilon(s,a,\tilde{r},\tilde{s}',\tilde{u},v_1,v_2) \gamma \big(v_1(\tilde{s}',\tilde{u})-v_2(\tilde{s}',\tilde{u})\big)\Big]\\
\overset{(c)}{\leq}& \max_a \mathbb{E}\Big[\big(1-\eta \xi^\epsilon(s,a,\tilde{r},\tilde{s}',\tilde{u},v_1,v_2)\big) ||v_1-v_2||_\infty + \eta \xi^\epsilon(s,a,\tilde{r},\tilde{s}',\tilde{u},v_1,v_2) \gamma ||v_1-v_2||_\infty\Big] \\
=& \max_a \mathbb{E}\Big[\big(1-\eta \xi^\epsilon(s,a,\tilde{r},\tilde{s}',\tilde{u},v_1,v_2)(1-\gamma)\big) ||v_1-v_2||_\infty\Big]\\
\leq& \big(1-\eta\epsilon\kappa (1-\gamma)\big) ||v_1-v_2||_\infty.
\end{aligned}
$$
$(a)$ is due to the fact that $\max_x f(x)-\max_x g(x)\leq \max_x (f(x)-g(x))$. $(b)$ is because  the derivative of $\partial l^\kappa_{\max(\epsilon,\min(1-\epsilon,\alpha))}(\cdot)$ satisfies $0<\epsilon\kappa \leq  \partial^2 l^\kappa_{\max(\epsilon,\min(1-\epsilon,\alpha))} \leq (1-\epsilon)\kappa^{-1}$. As a result, there exists $\xi \in[\epsilon\kappa, (1-\epsilon)\kappa^{-1}]$ such that $\partial l^{\kappa}_{\max(\epsilon,\min(1-\epsilon,\alpha))}(x)-\partial l^{\kappa}_{\max(\epsilon,\min(1-\epsilon,\alpha))}(y) = \xi (x-y)$. \textcolor{black}{Here the value of $\xi$ depends on $s,a,\tilde{r},\tilde{s}',\tilde{u}, v_1,v_2$, so we represent it as a function.} $(c)$ follows from $0< \eta\leq \kappa\leq \kappa/\max(\epsilon,(1-\epsilon))$ so that $\eta\xi^\epsilon(s,a,\tilde{r},\tilde{s}',\tilde{u},v_1,v_2)\in[0,1]$ almost surely.
By symmetry, we have
$$
|\mathcal{T}_{\epsilon,\kappa}^*v_1(s,\alpha)-\mathcal{T}_{\epsilon,\kappa}^* v_2(s,\alpha)|\leq \big(1-\eta \epsilon\kappa(1-\gamma)\big)||v_1-v_2||_\infty,
$$
Note that when $\epsilon= 0$, the contraction ratio will be $1$. This is the reason why we clip $\alpha\in[\epsilon, 1-\epsilon]$ for $l^\kappa_\alpha$. In practice, when discretizing the quantile range, e.g. as the case in QR-DQN~\cite{dabney2018distributional}, the quantile levels considered are guaranteed to lie in a range $[\epsilon, 1-\epsilon]$.

Choosing $0<\eta\leq \kappa \leq 1 \leq\frac{1}{\kappa\epsilon(1-\gamma)}$ yields a contraction. 
\end{proof}

\subsection{Proof of Proposition~\ref{prop:fix-point}}
\Propvstar*


\begin{proof}
Noticing that when $\epsilon= 0$ and $\kappa =0$, the soft quantile loss function $l^\kappa_{\epsilon,\alpha}(\cdot)=l^\kappa_{\max(\epsilon,\min(1-\epsilon,\alpha))}(\cdot)=l_\alpha(\cdot)$, we show that if $\tilde{r}(s,a)+\gamma  v^*(\tilde{s}',\tilde{u})$ has a density for all $v$ and $(s,a)$ (e.g. the random reward is composed of an independent additive term with a density), then $v^*$ is the unique fixed point of 
\begin{equation}
\mathcal{T}_{0,0}^* v(s,\alpha):=v(s,\alpha)+\eta\max_a \mathbb{E}[\partial l_\alpha(\tilde{r}(s,a)+\gamma  v(\tilde{s}',\tilde{u})-v(s,\alpha))].\label{eq:T00BellmanEq}    
\end{equation}
with $\partial l_\alpha(\delta):=\alpha-\mathbb{I}\{\delta<0\}$.

First, we show that $v^*$ satisfies Eq.~\ref{eq:T00BellmanEq}. This is equivalent to showing that for all $(s,a)$:    \begin{equation}
    \label{eq:proof-zero}
        \max_a \mathbb{E}\Big[\partial l_\alpha\big(\tilde{r}(s,a)+\gamma v^*(\tilde{s}',\tilde{u})-v^*(s,\alpha)\big)\Big]=0.
    \end{equation}
    According to Eq.~\ref{eq:qdp-nested}, we have 
    \begin{equation}
    \label{eq:var-optimal-v}
        v^*(s,\alpha)=\max_a q^*(s,\alpha,a)=\max_a \mathrm{VaR}_\alpha[\tilde{r}(s,a)+\gamma v^*(\tilde{s}',\tilde{u})],
    \end{equation}
which implies that 
    \begin{equation}
    \label{eq:v-star-geq}
        v^*(s,\alpha) \geq q^*(s,\alpha,a):=\mathrm{VaR}_\alpha[\tilde{r}(s,a) + \gamma v^*(\tilde{s}',\tilde{u})], ~\forall a,
    \end{equation}
Recalling the elicitability of quantiles, Eq. \ref{eq:v-star-geq} implies that for any $a$:
\[\arg
\min_y \mathbb{E}[l_\alpha(\tilde{r}(s,a)+\gamma v^*(\tilde{s}',\tilde{u})-y)] \leq \mathrm{VaR}_\alpha[\tilde{r}(s,a)+\gamma v^*(\tilde{s}',\tilde{u})]\leq v^*(s,\alpha),\]
which in turn, by convexity of $f(s,\alpha,a,y):=\mathbb{E}[l_\alpha(\tilde{r}(s,a)+\gamma v^*(\tilde{s}',\tilde{u})-y)]$ in $y$, lets us establish that:
    \begin{equation*}
        \nabla_y\mathbb{E}\Big[l_\alpha\big(\tilde{r}(s,a)+\gamma v^*(\tilde{s}',\tilde{u})-y\big)\Big]\Big|_{y=v^*(s,\alpha)} =\nabla_y f(s,\alpha,a,y)\Big|_{y=v^*(s,\alpha)} \geq \nabla_y f(s,\alpha,a,y)\Big|_{y=q^*(s,\alpha,a)}= 0, ~\forall a,
    \end{equation*}
where the gradient of $f(s,\alpha,a,y)$ exist for all $y$ since $\tilde{r}(s,a)+\gamma v^*(\tilde{s}',\tilde{u})$ is assumed to have a density.
    
As $\nabla_y l_\alpha(\tilde{x}-y)=-\partial l_\alpha(\tilde{x}-y)$ almost everywhere, we have
\begin{equation*}
    \mathbb{E}\Big[\partial l_\alpha\big(\tilde{r}(s,a)+\gamma v^*(\tilde{s}',\tilde{u})-v^*(s,\alpha)\big)\Big] = - \nabla_y\mathbb{E}\Big[l_\alpha\big(\tilde{r}(s,a)+\gamma v^*(\tilde{s}',\tilde{u})-y\big)\Big]\Big|_{y=v^*(s,\alpha)} \leq 0, ~\forall a
\end{equation*}
The zero value is achieved when action is the optimal action $a^*=\arg\max_a q^*(s,\alpha,a)$. Therefore, Eq.~\ref{eq:proof-zero} holds.

Second, we show that $v^*$ is unique. Take any $\bar{v}\neq v^*$ is another fixed point of Eq.~\ref{eq:T00BellmanEq}, we must have 
that 
\[\mathbb{E}\Big[\partial l_\alpha\big(\tilde{r}(s,\bar{\pi}(s,\alpha))+\gamma \bar{v}(\tilde{s}',\tilde{u})-\bar{v}(s,\alpha)\big)\Big]=0,\]
when $\bar{\pi}(s,\alpha):=\arg\max_a \mathbb{E}\Big[\partial l_\alpha\big(\tilde{r}(s,a)+\gamma \bar{v}(\tilde{s}',\tilde{u})-\bar{v}(s,\alpha)\big)\Big]$ thus implying that:
\[\bar{v}(s,\alpha) = \mathrm{VaR}_\alpha^{\bar{\pi}}[\tilde{r}(s,\bar{\pi}(s,\alpha))+\gamma \bar{v}(\tilde{s}',\tilde{u})].\]

One also has that $\bar{v}$ does not satisfy Eq.\ref{eq:var-optimal-v}, otherwise Eq.~\ref{eq:var-optimal-v} would have two fixed points. So there must exist a $(\bar{s},\bar{\alpha})$ such that
$$
\bar{v}(\bar{s},\bar{\alpha})\neq \max_a \mathrm{VaR}_\alpha[\tilde{r}(\bar{s},a) + \gamma \bar{v}(\tilde{s}',\tilde{u})]\geq \mathrm{VaR}_\alpha^{\bar{\pi}}[\tilde{r}(s,\bar{\pi}(s,\alpha))+\gamma \bar{v}(\tilde{s}',\tilde{u})] = \bar{v}(s,\alpha).
$$
This necessarily means that there exists $\bar{a}$ such that:
\[\bar{v}(\bar{s},\bar{\alpha})< \bar{q}(\bar{s},\bar{\alpha},\bar{a}):=\mathrm{VaR}_{\bar{\alpha}}[\tilde{r}(\bar{s},\bar{a}) + \gamma \bar{v}(\tilde{s}',\tilde{u})].\]
By convexity of $\bar{f}(s,\alpha,a,y):=\mathbb{E}[l_\alpha(\tilde{r}(s,a)+\gamma \bar{v}(\tilde{s}',\tilde{u})-y)]$ and uniqueness of the minimizer $\mathrm{VaR}_{\bar{\alpha}}[\tilde{r}(\bar{s},\bar{a}) + \gamma \bar{v}(\tilde{s}',\tilde{u})]=\arg\min_y \bar{f}(\bar{s},\bar{\alpha},\bar{a},y)$, we conclude that 
\[\nabla_y \bar{f}(\bar{s},\bar{\alpha},\bar{a},y)\Big|_{y=\bar{v}(\bar{s},\bar{\alpha})} < \nabla_y \bar{f}(\bar{s},\bar{\alpha},\bar{a},y)\Big|_{y=\bar{q}(\bar{s},\bar{\alpha},\bar{a})}=0.\]

This contradicts the fact that $\bar{v}$ satisfies  Eq.~\ref{eq:T00BellmanEq} since:
\[0<-\nabla_y \bar{f}(\bar{s},\bar{\alpha},\bar{a},y)\Big|_{y=\bar{v}(\bar{s},\bar{\alpha})}=\mathbb{E}\Big[\partial l_{\bar{\alpha}}\big(\tilde{r}(\bar{s},\bar{a})+\gamma \bar{v}(\tilde{s}',\tilde{u})-\bar{v}(\bar{s},\bar{\alpha})\big)\Big]\leq \max_a \mathbb{E}\Big[\partial l_{\bar{\alpha}}\big(\tilde{r}(\bar{s},a)+\gamma \bar{v}(\tilde{s}',\tilde{u})-\bar{v}(\bar{s},\bar{\alpha})\big)\Big],\]
which would imply that:
\[\mathcal{T}_{0,0}^*\bar{v}(\bar{s},\bar{\alpha})  = \bar{v}(\bar{s},\bar{\alpha})+\eta\max_a \mathbb{E}[\partial l_\alpha(\tilde{r}(\bar{s},a)+\gamma  \bar{v}(\tilde{s}',\tilde{u})-\bar{v}(\bar{s},\bar{\alpha}))] > \bar{v}(\bar{s},\bar{\alpha}) .\]
\end{proof}

\subsection{Analysis for $\mathcal{T}_{\epsilon,\kappa}^{\hat{\pi}}$}
\label{app:T-pi}
\begin{proposition}\label{thm:TpiContraction}
    $\mathcal{T}_{\epsilon,\kappa}^{\hat{\pi}}$ is a contraction mapping for $v$  with step size $\eta\in(0,\kappa]$.
\end{proposition}
\begin{proof}
The proof is similar to the proof of proposition~\ref{prop:contraction}.
$$
\begin{aligned}
&\mathcal{T}_{\epsilon,\kappa}^{\hat{\pi}}v_1(s,\alpha) - \mathcal{T}_{\epsilon,\kappa}^{\hat{\pi}} v_2(s,\alpha)\\
=& v_1(s,\alpha) + \eta \mathbb{E}\Big[\partial l^\kappa_{\max(\epsilon,\min(1-\epsilon,\alpha))}(\tilde{r}(s,\tilde{a})+\gamma v_1(\tilde{s}',\tilde{u})-v_1(s,\alpha))\Big] - v_2(s,\alpha)\\
&- \eta \mathbb{E}\Big[\partial l^\kappa_{\max(\epsilon,\min(1-\epsilon,\alpha))}(\tilde{r}(s,\tilde{a})+\gamma v_2(\tilde{s}',\tilde{u})-v_2(s,\alpha))\Big]\\
=& v_1(s,\alpha)-v_2(s,\alpha) +  \eta\mathbb{E}\Big[\partial l^\kappa_{\max(\epsilon,\min(1-\epsilon,\alpha))}(\tilde{r}(s,\tilde{a})+\gamma v_1(\tilde{s}',\tilde{u})-v_1(s,\alpha))\\
&-\partial l^\kappa_{\max(\epsilon,\min(1-\epsilon,\alpha))}(\tilde{r}(s,\tilde{a})+\gamma v_2(\tilde{s}',\tilde{u})-v_2(s,\alpha))\Big]\\
\overset{(a)}{=}&v_1(s,\alpha)-v_2(s,\alpha) + \eta\mathbb{E}\Big[\xi^\epsilon(s,\tilde{a},\tilde{r},\tilde{s}',\tilde{u},v_1,v_2)(\gamma v_1(\tilde{s}',\tilde{u})-\gamma v_2(\tilde{s}',\tilde{u})-v_1(s,\alpha)+v_2(s,\alpha))\Big], \\
&\epsilon\kappa \leq \xi^\epsilon(s,\tilde{a},\tilde{r},\tilde{s}',\tilde{u}, v_1, v_2) \leq (1-\epsilon)\kappa^{-1}\\
=& \mathbb{E}\Big[\big(1-\eta \xi^\epsilon(s,\tilde{a},\tilde{r},\tilde{s}',\tilde{u},v_1,v_2)\big) (v_1(s,\alpha)-v_2(s,\alpha)) + \eta \xi^\epsilon(s,\tilde{a},\tilde{r},\tilde{s}',\tilde{u},v_1,v_2) \gamma \big(v_1(\tilde{s}',\tilde{u})-v_2(\tilde{s}',\tilde{u})\big)\Big]\\
\overset{(b)}{\leq}&  \mathbb{E}\Big[\big(1-\eta \xi^\epsilon(s,\tilde{a},\tilde{r},\tilde{s}',\tilde{u},v_1,v_2)\big) ||v_1-v_2||_\infty + \eta \xi^\epsilon(s,\tilde{a},\tilde{r},\tilde{s}',\tilde{u},v_1,v_2) \gamma ||v_1-v_2||_\infty\Big] \\
=&  \mathbb{E}\Big[\big(1-\eta \xi^\epsilon(s,\tilde{a},\tilde{r},\tilde{s}',\tilde{u},v_1,v_2)(1-\gamma)\big) ||v_1-v_2||_\infty\Big]\\
\leq& \big(1-\eta\epsilon\kappa (1-\gamma)\big) ||v_1-v_2||_\infty.
\end{aligned}
$$
$(a)$ is because  the derivative of $\partial l^\kappa_{\max(\epsilon,\min(1-\epsilon,\alpha))}(\cdot)$ satisfies $0<\epsilon\kappa \leq  \partial^2 l^\kappa_{\max(\epsilon,\min(1-\epsilon,\alpha))} \leq (1-\epsilon)\kappa^{-1}$. As a result, there exists $\xi \in[\epsilon\kappa, (1-\epsilon)\kappa^{-1}]$ such that $\partial l^{\kappa}_{\max(\epsilon,\min(1-\epsilon,\alpha))}(x)-\partial l^{\kappa}_{\max(\epsilon,\min(1-\epsilon,\alpha))}(y) = \xi (x-y)$. {Here the value of $\xi$ depends on $s,a,\tilde{r},\tilde{s}',\tilde{u}, v_1,v_2$, so we represent it as a function.} $(b)$ follows from $0< \eta\leq \kappa\leq \kappa/\max(\epsilon,(1-\epsilon))$ so that $\eta\xi^\epsilon(s,\tilde{a},\tilde{r},\tilde{s}',\tilde{u},v_1,v_2)\in[0,1]$ almost surely.
By symmetry, we have
$$
|\mathcal{T}_{\epsilon,\kappa}^{\hat{\pi}}v_1(s,\alpha)-\mathcal{T}_{\epsilon,\kappa}^{\hat{\pi}} v_2(s,\alpha)|\leq \big(1-\eta \epsilon\kappa(1-\gamma)\big)||v_1-v_2||_\infty,
$$
Choosing $0<\eta\leq \kappa \leq 1 \leq\frac{1}{\kappa\epsilon(1-\gamma)}$ yields a contraction. 
\end{proof}

Next, similar to Proposition~\ref{prop:fix-point}, we show that the fixed point of Eq.~\ref{eq:var-pe} is the same as the fixed point of  Eq.~\ref{eq:qdp-nested3} when $\epsilon=0$ and $\kappa= 0$.

\begin{proposition}\label{prop:BellPolicyEval}
The Bellman equation
\begin{equation}
\label{eq:qdp-nested3}
    v(s,\alpha)=\mathrm{VaR}_\alpha^{\hat{\pi}}[\tilde{r}(s,\tilde{a}) + \gamma v(\tilde{s}',\tilde{u})]
\end{equation}
    where $\tilde{a}\sim\hat{\pi}(s,\alpha)$ has a unique fixed, denoted as  $\mathfrak{v}^{\hat{\pi}}$. Moreover, if $\hat{\pi}$ is \textcolor{black}{independent of} $\alpha$, then $\mathfrak{v}^{\hat{\pi}}(s,\alpha)=\mathrm{VaR}_\alpha^{\hat{\pi}}[\sum_{t=0}^{\infty}\tilde{r}(\tilde{s}_t,\tilde{a}_t)|\tilde{s}_0=s]$, where $\tilde{a}_t\sim\hat{\pi}(\tilde{s}_t,\bar{\alpha})$, \textcolor{black}{and $\hat{\pi}(s,\bar{\alpha})$ is identical for any $\bar{a}\in[0,1]$.}
\end{proposition}

\begin{proof}
    We first demonstrate that the Bellman operator
    \[\mathcal{B}^{\hat{\pi}}v(s,\alpha) := \mathrm{VaR}_\alpha^{\hat{\pi}}[\tilde{r}(s,\tilde{a}) + \gamma v(\tilde{s}',\tilde{u})]\]
    is a $\gamma$-contraction according to the sup-norm and therefore admits a unique fixed point that we denote as $\mathfrak{v}^{\hat{\pi}}$. Namely, take any two function $v_1$ and $v_2$. One can verify
    \begin{align*}
        \mathcal{B}^{\hat{\pi}}&v_1(s,\alpha) - \mathcal{B}^{\hat{\pi}}v_2(s,\alpha) \\
        &= \mathrm{VaR}_\alpha^{\hat{\pi}}[\tilde{r}(s,\tilde{a}) + \gamma v_1(\tilde{s}',\tilde{u})] - \mathrm{VaR}_\alpha^{\hat{\pi}}[\tilde{r}(s,\tilde{a}) + \gamma v_2(\tilde{s}',\tilde{u})]\\
        &\leq \mathrm{VaR}_\alpha^{\hat{\pi}}[ \tilde{r}(s,\tilde{a}) + \gamma(v_2(\tilde{s}',\tilde{u})+\|v_1-v_2\|_\infty)]\\
        &\qquad\qquad- \mathrm{VaR}_\alpha^{\hat{\pi}}[ \tilde{r}(s,\tilde{a}) + \gamma v_2(\tilde{s}',\tilde{u})]\\
&=\gamma  \|v_1-v_2\|_\infty.
\end{align*}

    Next, we consider the case where $\hat{\pi}$ is independent of $\alpha$. For any $T>0$, one can show that the finite horizon value function:
    \begin{align*}
    v_T^{\hat{\pi}}(s,\alpha)&:= \mathrm{VaR}^{\hat{\pi}}_\alpha[\sum_{t=0}^{T-1}\gamma^t\tilde{r}(\tilde{s}_t,\tilde{a}_t)|\tilde{s}_0=s]\\
    &\overset{(a)}{=}\mathrm{VaR}^{\hat{\pi}}_\alpha[\mathrm{VaR}^{\hat{\pi}}_{\tilde{u}}[\sum_{t=0}^{T-1}\gamma^t\tilde{r}(\tilde{s}_t,\tilde{a}_t)|\tilde{s}_{0:1},\tilde{a}_{0},\tilde{r}(\tilde{s}_0,\tilde{a}_0)]|\tilde{s}_0=s]\\
    &=\mathrm{VaR}^{\hat{\pi}}_\alpha[\tilde{r}(s,\tilde{a})+\gamma\mathrm{VaR}^{\hat{\pi}}_{\tilde{u}}[\sum_{t=1}^{T-1}\gamma^{t-1}\tilde{r}(\tilde{s}_t,\tilde{a}_t)|\tilde{s}_{0:1},\tilde{a}_{0},\tilde{r}(\tilde{s}_0,\tilde{a}_0)]|\tilde{s}_0=s]\\
    &=\mathrm{VaR}^{\hat{\pi}}_\alpha[\tilde{r}(s,\tilde{a})+\gamma\mathrm{VaR}^{\hat{\pi}}_{\tilde{u}}[\sum_{t=1}^{T-1}\gamma^{t-1}\tilde{r}(\tilde{s}_t,\tilde{a}_t)|\tilde{s}_{1}]|\tilde{s}_0=s]\\    
    &=\mathcal{B}^{\hat{\pi}}v_{T-1}^{\hat{\pi}}(s,\alpha),
    \end{align*}
where (a) follows from $X$ having the same distribution as $\mathrm{VaR}_{\tilde{u}}[X|Y]$ due to the inverse probability integral transform theorem, applied conditionally. This allows us to conclude that:
    \[v_T^{\hat{\pi}}(s,\alpha) = (\mathcal{B}^{\hat{\pi}})^T v_0(s,\alpha)\]
    with $v_0(s,\alpha):=0$. The $\gamma$-contraction property of $\mathcal{B}$ implies that:
    \[\|v_T^{*}-\mathfrak{v}^{\hat{\pi}}\|_\infty = \|(\mathcal{B}^{\hat{\pi}})^T v_0-(\mathcal{B}^{\hat{\pi}})^T\mathfrak{v}^{\hat{\pi}}\|_\infty \leq \gamma^T\|v_0-\mathfrak{v}^{\hat{\pi}}\|_\infty = \gamma^T\|\mathfrak{v}^{\hat{\pi}}\|_\infty.\]
    This confirms that  $ \lim_{T \rightarrow \infty}v_T^{\hat{\pi}}(s,\alpha)=\mathfrak v^{{\pi}}(s,\alpha)$, which means that $\mathfrak{v}^{\hat{\pi}}(s,\alpha)=\mathrm{VaR}_\alpha^{\hat{\pi}}[\sum_{t=0}^\infty \tilde{r}(\tilde{s}_t,\tilde{a}_t)|\tilde{s}_0=s]$.
\end{proof}

\begin{proposition}
    The quantile value $\mathfrak{v}^{\hat{\pi}}(s,\alpha)$ identified in Proposition \ref{prop:BellPolicyEval}
    is the unique fixed point of $\mathcal{T}_{\epsilon,\kappa}^{\hat{\pi}}$ when $\epsilon=0$ and $\kappa= 0$.    
\end{proposition}

\begin{proof}
    Noticing that when $\epsilon= 0$ and $\kappa =0$, the soft quantile loss function $l^\kappa_{\epsilon,\alpha}(\cdot)=l^\kappa_{\max(\epsilon,\min(1-\epsilon,\alpha))}(\cdot)=l_\alpha(\cdot)$, we show that if $\tilde{r}(s,a)+\gamma  v(\tilde{s}',\tilde{u})$ has a  density and unique quantile for all $v$ and $(s,a)$ (e.g. the random reward is composed of an independent additive white noise), then $\mathfrak{v}^{\hat{\pi}}$ is the unique fixed point of 
\begin{equation}
\mathcal{T}_{0,0}^{\hat{\pi}} v(s,\alpha):=v(s,\alpha)+\eta\mathbb{E}[\partial l_\alpha(\tilde{r}(s,\tilde{a})+\gamma  v(\tilde{s}',\tilde{u})-v(s,\alpha))],\label{eq:T00BellmanEqPi}    
\end{equation}
with $\partial l_\alpha(\delta):=\alpha-\mathbb{I}\{\delta<0\}$.

First, we show that $\mathfrak{v}^{\hat{\pi}}$ satisfies Eq.~\ref{eq:T00BellmanEqPi}. This is equivalent to showing that for all $(s,\alpha)$:    \begin{equation}
    \label{eq:proof-zero-pi}
        \mathbb{E}\Big[\partial l_\alpha\big(\tilde{r}(s,\tilde{a})+\gamma \mathfrak{v}^{\hat{\pi}}(\tilde{s}',\tilde{u})-\mathfrak{v}^{\hat{\pi}}(s,\alpha)\big)\Big]=0.
    \end{equation}
    According to Proposition \ref{prop:BellPolicyEval} and the elicitability of quantiles, we have 
    \begin{equation}
    \label{eq:var-pi-v}
        \mathfrak{v}^{\hat{\pi}}(s,\alpha)=\mathrm{VaR}_\alpha[\tilde{r}(s,\tilde{a})+\gamma \mathfrak{v}^{\hat{\pi}}(\tilde{s}',\tilde{u})]=\arg
\min_y \mathbb{E}[l_\alpha(\tilde{r}(s,\tilde{a})+\gamma \mathfrak{v}^{\hat{\pi}}(\tilde{s}',\tilde{u})-y)],
    \end{equation}
    due to the assumed uniqueness of the quantile of $\tilde{r}(s,\tilde{a})+\gamma \mathfrak{v}^{\hat{\pi}}(\tilde{s}',\tilde{u})$.  In turn, by convexity of $f(s,\alpha,y):=\mathbb{E}[l_\alpha(\tilde{r}(s,\tilde{a})+\gamma \mathfrak{v}^{\hat{\pi}}(\tilde{s}',\tilde{u})-y)]$ in $y$, this lets us establish that:
    \begin{equation*}
        0 = \nabla_y f(s,\alpha,y)=\nabla_y\mathbb{E}\Big[l_\alpha\big(\tilde{r}(s,\tilde{a})+\gamma \mathfrak{v}^{\hat{\pi}}(\tilde{s}',\tilde{u})-y\big)\Big]\Big|_{y=\mathfrak{v}^{\hat{\pi}}(s,\alpha)} ,
    \end{equation*}
where the derivative of $f(s,\alpha,a,y)$ exist for all $y$ since $\tilde{r}(s,a)+\gamma \mathfrak{v}^*(\tilde{s}',\tilde{u})$ is assumed to have a density.
    
As $\nabla_y l_\alpha(\tilde{x}-y)=-\partial l_\alpha(\tilde{x}-y)$ almost everywhere, we have
\begin{equation*}
    \mathbb{E}\Big[\partial l_\alpha\big(\tilde{r}(s,\tilde{a})+\gamma \mathfrak{v}^{\hat{\pi}}(\tilde{s}',\tilde{u})-\mathfrak{v}^{\hat{\pi}}(s,\alpha)\big)\Big] = - \nabla_y\mathbb{E}\Big[l_\alpha\big(\tilde{r}(s,\tilde{a})+\gamma \mathfrak{v}^{\hat{\pi}}(\tilde{s}',\tilde{u})-y\big)\Big]\Big|_{y=\mathfrak{v}^{\hat{\pi}}(s,\alpha)} = 0.
\end{equation*}
Therefore, Eq.~\ref{eq:proof-zero-pi} holds.

Second, we show that $\mathfrak{v}^{\hat{\pi}}$ is unique. Take any $\bar{\mathfrak v}\neq \mathfrak v^{\hat{\pi}}$ is another fixed point of Eq.~\ref{eq:T00BellmanEqPi}, we must have 
that 
\[\mathbb{E}\Big[\partial l_\alpha\big(\tilde{r}(s,\tilde{a})+\gamma \bar{\mathfrak v}(\tilde{s}',\tilde{u})-\bar{\mathfrak v}(s,\alpha)\big)\Big]=0\]
thus implying, due to $\tilde{r}(s,a)+\gamma \bar{\mathfrak v}(\tilde{s}',\tilde{u})$ having a  density, that:
\[\bar{\mathfrak v}(s,\alpha) = \mathrm{VaR}_\alpha^{\hat{\pi}}[\tilde{r}(s,\tilde{a})+\gamma \bar{\mathfrak v}(\tilde{s}',\tilde{u})].\]
This leads to a contradiction since Proposition \ref{prop:BellPolicyEval} established that Eq. \ref{eq:qdp-nested3} has a unique fixed point.
\end{proof}

\subsection{Analysis for VaR Policy Gradient}
\label{app:var-pg}
We first make the assumption below.
\begin{assumption}
\label{ass:A6}
    1) The space of  $\mathcal{A}$ is finite. 2) Reward is bounded, i.e., there exits $\bar{r}>0$ such that $|\tilde{r}(s,a)|\leq \bar{r}, \forall(s,a)$. 
\end{assumption}


Consider the VaR policy gradient, i.e.,
\begin{equation}
\label{eq:var-policy-gradient}
\theta_{k+1}(s,\alpha)\leftarrow \theta_k(s,\alpha) +  \eta \cdot \mathbb{E}_{a\sim \hat{\pi}_{\theta_k}(s,\alpha)}\Big[\mathbb{E}\big[\partial l^\kappa_{\epsilon,\alpha}\big(\delta_{{v}^{\hat{\pi}_k}}(s,\alpha,\tilde{r}(s,a),\tilde{s}',\tilde{u}_d)\big)\big] \nabla_\theta \log \hat\pi_{\theta_k}(a|s,\alpha)\Big], \forall s,\alpha,
\end{equation}
where $\theta_{k}$ is the parameter of $\hat\pi_{\theta_k}$, $v^{\hat{\pi}_k}$ is the value function under $\hat\pi_{\theta_k}$ (i.e., the fixed point of $\mathcal{T}^{\hat\pi_k}_{\epsilon,\kappa}$), and $a$ is sampled from $\hat\pi_{\theta_k}(\cdot|s,\alpha)$. We show that with proper learning rate, this policy update guarantees monotone increasing of $v^{\hat{\pi}_k}$, if the policy parameterization is softmax. The analysis follows similar ideas as \citet{agarwal2021theory}. 

We start with the new risk measure induced by the soft function $l^\kappa_{\epsilon,\alpha}$ and show that another Bellman operator using this new risk measure is monotone and has the same unique fixed point as $\mathcal{T}^{\hat\pi}_{\epsilon,\kappa}$.

Define 
\[
\rho_{\epsilon,\alpha}^\kappa(X):= \arg\min_x \mathbb{E}[\ell_{\epsilon,\alpha}^\kappa(X-x)],
\]
with $\ell_{\epsilon,\alpha}^\kappa(x):=\ell_{\max(\epsilon,\min(1-\epsilon,\alpha))}^\kappa(x)$.

\begin{lemma}
    The risk measure $\rho_{\epsilon,\alpha}^\kappa$ is normalized, monotone, and translation invariant.
\end{lemma}
\begin{proof}
    These are well known properties of utility-based shortfall risk measures. In order, normalization follows from:
    \[\rho_{\epsilon,\alpha}^\kappa(0)=\arg\min_x \mathbb{E}[\ell_{\epsilon,\alpha}^\kappa(0-x)] = 0,\]
    since $\ell_{\epsilon,\alpha}^\kappa(x)\geq 0$ for all $x$ with only equality at $x=0$.

    Translation invariance follows from:
    \[\rho_{\epsilon,\alpha}^\kappa(X+t)=\arg\min_x \mathbb{E}[\ell_{\epsilon,\alpha}^\kappa(X+t-x)]=\arg\min_{y}\mathbb{E}[\ell_{\epsilon,\alpha}^\kappa(X-y)]+t=\rho_{\epsilon,\alpha}^\kappa(X)+t,\]
    where we applied the change of variables $y:= x-t$.

    Finally, monotonicity follows from the fact that, given any random $X$ and random $\Delta\geq 0$:
    \begin{align*}
        \nabla_x \mathbb{E}[\ell_{\epsilon,\alpha}^\kappa(X+\Delta-x)]\Big|_{x=\rho_{\alpha}^\kappa(X)}&= -\mathbb{E}\left[\partial\ell_{\epsilon,\alpha}^\kappa(X+\Delta-\rho_{\alpha}^\kappa(X))\right]\\
        &\leq -\mathbb{E}\left[\partial\ell_{\epsilon,\alpha}^\kappa(X-\rho_{\alpha}^\kappa(X))\right]\\
        &=\nabla_x \mathbb{E}[\ell_{\epsilon,\alpha}^\kappa(X-x)]\Big|_{x=\rho_{\alpha}^\kappa(X)}=0,
    \end{align*}
    due to the non-decreasingness of $\partial\ell_{\epsilon,\alpha}^\kappa(y)$.
    
By convexity of $\mathbb{E}[\ell_{\epsilon,\alpha}^\kappa(X+\Delta-x)]$ in $x$, this implies that     
\[  \rho_{\epsilon,\alpha}^\kappa(X)\leq \arg\min_x \mathbb{E}[\ell_{\epsilon,\alpha}^\kappa(X+\Delta-x)]=\rho_{\epsilon,\alpha}^\kappa(X+\Delta).\]
\end{proof}

Then consider the Bellman equation with the risk $\rho^\kappa_{\epsilon,\alpha}$, i.e.,
\[
\mathcal{B}_{\epsilon,\kappa}^{\hat{\pi}} v(s,\alpha) := \rho_{\epsilon,\alpha}^\kappa (\tilde{r}(s,\tilde{a})+\gamma v(\tilde{s}',\tilde{u})),~~\tilde{a}\sim \hat{\pi}(s,\alpha).
\]

\begin{lemma}
    The operator $\mathcal{B}_{\epsilon,\kappa}^{\hat{\pi}}$ is monotone and a contraction with the same unique fixed point as $\mathcal{T}^{\hat{\pi}}_{\epsilon,\kappa}$.
\end{lemma}
\begin{proof}
    We start with proving monotonicity by exploiting the monotonicity of $\rho_{\epsilon,\alpha}^\kappa$. Namely, if $v_1\geq v_2$, then:
    \[\mathcal{B}_{\epsilon,\kappa}^{\hat{\pi}} v_1(s,\alpha) = \rho_{\epsilon,\alpha}^\kappa (\tilde{r}(s,\tilde{a})+\gamma v_1(\tilde{s}',\tilde{u}))\geq \rho_{\epsilon,\alpha}^\kappa (\tilde{r}(s,\tilde{a})+\gamma v_2(\tilde{s}',\tilde{u})) = \mathcal{B}_{\epsilon,\kappa}^{\hat{\pi}} v_2(s,\alpha).\]
    
    We follow with proving contraction, which follows similarly as in the proof of Proposition \ref{thm:BfixedPoint}. Namely, using monotonicity and translation invariance of $\rho_{\epsilon,\alpha}^\kappa$, we get:
    \begin{align*}
        \mathcal{B}_{\epsilon,\kappa}^{\hat{\pi}}&v_1(s,\alpha) - \mathcal{B}_{\epsilon,\kappa}^{\hat{\pi}}v_2(s,\alpha) \\
        &= \rho_{\epsilon,\alpha}^\kappa(\tilde{r}(s,\tilde{a}) + \gamma v_1(\tilde{s}',\tilde{u})) - \rho_{\epsilon,\alpha}^\kappa(\tilde{r}(s,\tilde{a}) + \gamma v_2(\tilde{s}',\tilde{u}))\\
        &\leq \rho_{\epsilon,\alpha}^\kappa( \tilde{r}(s,\tilde{a}) + \gamma(v_2(\tilde{s}',\tilde{u})+\|v_1-v_2\|_\infty))- \rho_{\epsilon,\alpha}^\kappa( \tilde{r}(s,\tilde{a}) + \gamma v_2(\tilde{s}',\tilde{u}))\\
        &=\gamma  \|v_1-v_2\|_\infty.
    \end{align*}    
    Now, to show the unique fixed point of $\mathcal{B}_{\epsilon,\kappa}^{\hat{\pi}}$ is the same as $\mathcal{T}_{\epsilon,\kappa}^{\hat{\pi}}$, we exploit the property that ${v}^{\hat{\pi}}=\mathcal{T}_{\epsilon,\kappa}^{\hat{\pi}}  {v}^{\hat{\pi}}$ which implies that 
    \[0 = (1/\eta)(\mathcal{T}_{\epsilon,\kappa}^{\hat{\pi}} {v}^{\hat{\pi}} - {v}^{\hat{\pi}})(s,\alpha) = \mathbb{E}_{\tilde{a}\sim\hat{\pi}}[\partial l^\kappa_{\epsilon,\alpha}(\delta_{{v}^{\hat{\pi}}}(s,\alpha,\tilde{r}(s,a),\tilde{s}',\tilde{u}))]=-\nabla_x \mathbb{E}[\ell_{\epsilon,\alpha}^\kappa(\tilde{r}(s,\tilde{a})+\gamma {v}^{\hat{\pi}}(\tilde{s}',\tilde{u})-x)]\Big|_{x={v}^{\hat{\pi}}(s,\alpha)}.\]
    Hence, 
    \[{v}^{\hat{\pi}}(s,\alpha)=\arg\min_x \mathbb{E}[\ell_{\epsilon,\alpha}^\kappa(\tilde{r}(s,\tilde{a})+\gamma {v}^{\hat{\pi}}(\tilde{s}',\tilde{u})-x)] = \rho_{\epsilon,\alpha}^\kappa(\tilde{r}(s,\tilde{a})+\gamma {v}^{\hat{\pi}}(\tilde{s}',\tilde{u}))=\mathcal{B}_{\epsilon,\kappa}^{\hat{\pi}}{v}^{\hat{\pi}}(s,\alpha).\]
\end{proof}

Next we show that the advantage function, i.e., $A^{\hat\pi_k}(s,\alpha,a):=\mathbb{E}\big[\partial l^\kappa_{\epsilon,\alpha}\big(\delta_{v^{\hat{\pi}_k}}(s,\alpha,\tilde{r}(s,a),\tilde{s}',\tilde{u})\big)\big] $ is bounded, and that with proper learning rate, $\mathbb{E}_{\tilde{a}\sim\hat{\pi}_{k+1}}[A^{\hat{\pi}_k}(s,\alpha,\tilde{a})]\geq \mathbb{E}_{\tilde{a}\sim\hat{\pi}_k}[A^{\hat{\pi}_k}(s,\alpha,\tilde{a})]$ for all $(s,\alpha)$.

\begin{lemma}\label{thm:boundA}
    With assumption~\ref{ass:A6} (specifically, $|\tilde{r}(s,a)|\leq \bar{r}, \forall (s,a)$), for any $\hat{\pi}$, the function $A^{\hat{\pi}}(s,\alpha,a):=\mathbb{E}[\partial\ell_{\epsilon,\alpha}^\kappa(\tilde{r}(s,a)+\gamma {v}^{\hat{\pi}}(\tilde{s}',\tilde{u})-{v}^{\hat{\pi}}(s,\alpha))]$ is such that $\|A^{\hat{\pi}}\|_\infty \leq 2\kappa\bar{r}/(1-\gamma)+1$.
\end{lemma}

\begin{proof}
    \textbf{First step:} We first show that $\|{v}^{\hat{\pi}}\|_\infty \leq \bar{r}/(1-\gamma)$. By induction on $T=0,1,\dots$, we will show that $(\mathcal{B}_{\epsilon,\kappa}^{\hat{\pi}})^T v_0(s,\alpha) \in [-\bar{r}/(1-\gamma), \bar{r}/(1-\gamma)]$, with $v_0(s,\alpha)=0$. This obviously applies at $T=0$. Now given that it applies at $T-1$, we can show that:
    \begin{align*}
    (\mathcal{B}_{\epsilon,\kappa}^{\hat{\pi}})^T v_0(s,\alpha) &= \mathcal{B}_{\epsilon,\kappa}^{\hat{\pi}}(\mathcal{B}_{\epsilon,\kappa}^{\hat{\pi}})^{T-1}v_0(s,\alpha)\leq \mathcal{B}_{\epsilon,\kappa}^{\hat{\pi}}(\bar{r}/(1-\gamma))(s,\alpha)\\
    &=\rho_{\epsilon,\alpha}^\kappa(\tilde{r}(s,\tilde{a})+\gamma\bar{r}/(1-\gamma))\leq \rho_{\epsilon,\alpha}^\kappa(\bar{r}+\gamma\bar{r}/(1-\gamma)) =\rho_{\epsilon,\alpha}^\kappa(\bar{r}/(1-\gamma)) \\
    &= \bar{r}/(1-\gamma),
    \end{align*}
where we in order exploit the monotonicity of $\mathcal{B}_{\epsilon,\kappa}^{\hat{\pi}}$ and of $\rho_{\epsilon,\alpha}^\kappa$, before finalizing with the normalization and translation invariance of $\rho_{\epsilon,\alpha}^\kappa$. 

An exactly similar argument can be used to show that $(\mathcal{B}_{\epsilon,\kappa}^{\hat{\pi}})^T v_0(s,\alpha)\geq -\bar{r}/(1-\gamma)$. Given that $\mathcal{B}_{\epsilon,\kappa}^{\hat{\pi}}$ is a contraction, we must have that:
    \[\bar{r}/(1-\gamma)\geq \lim_{T\rightarrow\infty} (\mathcal{B}_{\epsilon,\kappa}^{\hat{\pi}})^T v_0 = \bar{v}^{\hat{\pi}} \geq -\bar{r}/(1-\gamma).\]

\textbf{Second step:} We then show that $|\partial\ell_{\epsilon,\alpha}^\kappa(y)|\leq \kappa|y|+1$. Consider $y\in [0,\kappa)$. We can see that with $\max(\epsilon,\min(1-\epsilon,\alpha))\leq 1$:
    \[0\leq \partial\ell_{\epsilon,\alpha}^\kappa(y) \leq y/\kappa\leq  y/\kappa + (1-\kappa^2)(\kappa-y)/\kappa = \kappa y - \kappa^2 + 1 \leq \kappa|y|+1\]
    where the first inequality comes from $\kappa^2\leq 1$ and $y\leq \kappa$. The case with $y\geq \kappa$ is also covered above. 
    
    Consider now $y\in(-\kappa,0)$, we can see similarly that:
    \[0\geq \partial\ell_{\epsilon,\alpha}^\kappa(y) \geq y/\delta \geq y/\delta - (1-\kappa^2)(\kappa+y)/\kappa = \kappa y + \kappa^2 - 1 \geq -\kappa|y|-1, \]
since $y\geq \kappa$. From this, we conclude that $|\partial\ell_{\epsilon,\alpha}^\kappa(y)|\leq \kappa|y|+1$. 

    \textbf{Third step:} We can conclude that:
    \begin{align*}
 |A^{\hat{\pi}}(s,\alpha,a)|&\leq \mathbb{E}[|\partial\ell_{\epsilon,\alpha}^\kappa(\tilde{r}(s,a)+\gamma {v}^{\hat{\pi}}(\tilde{s}',\tilde{u})-{v}^{\hat{\pi}}(s,\alpha))|]\\
 &\leq \mathbb{E}[\kappa|\tilde{r}(s,a)+\gamma {v}^{\hat{\pi}}(\tilde{s}',\tilde{u})-{v}^{\hat{\pi}}(s,\alpha)|+1]\\
 &\leq \mathbb{E}[\kappa(\bar{r}+\gamma \bar{r}/(1-\gamma) + \bar{r}/(1-\gamma))+1]\\
 &=2\kappa\bar{r}/(1-\gamma)+1
    \end{align*}
\end{proof}

\begin{lemma}\label{thm:monotoneConv1}
    With assumption~\ref{ass:A6}, let $\hat{\pi}_\theta$ be a softmax policy, the update rule in Eq.~\ref{eq:var-policy-gradient} with $\eta\leq 1/(5\|A^{\hat{\pi}_k}\|_\infty)$ ensures that $\mathbb{E}_{\tilde{a}\sim\hat{\pi}_{k+1}}[A^{\hat{\pi}_k}(s,\alpha,\tilde{a})]\geq \mathbb{E}_{\tilde{a}\sim\hat{\pi}_k}[A^{\hat{\pi}_k}(s,\alpha,\tilde{a})]$ for all $(s,\alpha)$.
\end{lemma}

\begin{proof}
    Let $F_{s,\alpha}^A(\theta) := \sum_a \pi_{\theta}(a|s) A(s,\alpha,a)$ for some $A$. By Lemma 52 in \citet{agarwal2021theory}, we have that 
    \[
    \|\nabla_{\theta(s,\alpha)}F_{s,\alpha}^A(\theta)-\nabla_{\theta'(s,\alpha)}F_{s,\alpha}^A(\theta')\|_2\leq 5\|A\|_\infty\|\theta(s,\alpha,\cdot)-\theta'(s,\alpha,\cdot)\|_2 .
    \]
    Hence, $F_{s,\alpha}^A(\theta)$ is $\beta:=5\|A\|_\infty$ smooth with respect to $\theta(s,\alpha)$.

    We further observe that the policy update can be rewritten as:
    \[\theta_{k+1}(s,\alpha,\cdot) := \theta_k(s,\alpha,\cdot)+\eta\nabla_{\theta(s,\alpha)} F_{s,\alpha}^{A^{\hat{\pi}_k}}(\theta_k),\;\forall s,\alpha. \]
    By Theorem 57 in \citet{agarwal2021theory}, we get that if $\eta\leq 1/(5\|A^{\hat{\pi}_k}\|_\infty)$, then
    $F_{s,\alpha}^{A^{\hat{\pi}_k}}(\theta_{k+1})\geq F_{s,\alpha}^{A^{\hat{\pi}_k}}(\theta_{k})$ so that 
    \[ \sum_a \hat{\pi}_{k+1}(a|s) A^{\hat{\pi}_k}(s,\alpha,a) \geq \sum_a \hat{\pi}_{k}(a|s,\alpha) A^{\hat{\pi}_k}(s,\alpha,a),\;\forall s,\alpha.\]
\end{proof}

Finally, we are able to prove the monotone increasing of $v^{\hat\pi_k}$.

\begin{proposition}\label{thm:mononoteConvergence}
With assumption~\ref{ass:A6}, let $\hat{\pi}_\theta$ be a softmax policy, the update rule in Eq.~\ref{eq:var-policy-gradient} with $\eta< 1/(5(2\kappa\bar{r}/(1-\gamma)+1))$ ensures that ${v}^{\hat\pi_{k+1}}\geq {v}^{\hat{\pi}_k}$.
\end{proposition}
\begin{proof}
Given any $k$, Lemma \ref{thm:boundA} and \ref{thm:monotoneConv1} ensure that  
    \[\mathbb{E}_{\tilde{a}\sim\hat{\pi}_{k+1}}[A^{\hat{\pi}_k}(s,\alpha,\tilde{a})]\geq \mathbb{E}_{\tilde{a}\sim\hat{\pi}_k}[A^{\hat{\pi}_k}(s,\alpha,\tilde{a})].\]
    We can further observe that since $v^{\hat{\pi}_k}$ is the fixed point of  $\mathcal{T}_{\epsilon,\kappa}^{\hat{\pi}_k}$ and $A^{\hat{\pi}_k}(s,\alpha,a) = \mathbb{E}[\partial l^\kappa_{\epsilon,\alpha}(\delta_{v^{\hat{\pi}_k}}(s,\alpha,\tilde{r}(s,a),\tilde{s}',\tilde{u}))]$, it must be that:
    \[\mathbb{E}_{\tilde{a}\sim\hat{\pi}_k}[A^{\hat{\pi}_k}(s,\alpha,\tilde{a})] = (1/\eta)(\mathcal{T}_{\epsilon,\kappa}^{\hat{\pi}_k} v^{\hat{\pi}_k} - {v}^{\hat{\pi}_k})(s,\alpha) = 0,\]
    hence that $\mathbb{E}_{\tilde{a}\sim\hat{\pi}_{k+1}}[A^{\hat{\pi}_k}(s,\alpha,\tilde{a})] \geq 0$ for all $(s,\alpha)$.

    Now, let $F_{s,\alpha}^{k}(x):=\mathbb{E}_{\tilde{a}\sim\hat{\pi}_{k+1}(s,\alpha)}[\ell_{\epsilon,\alpha}^\kappa(\tilde{r}(s,\tilde{a})+\gamma {v}^{\hat{\pi}_k}(\tilde{s}',\tilde{u})-x)]$, a convex function of $x$, we can check that:
    \begin{align*}
    \nabla_x F_{s,\alpha}^{k}({v}^{\hat{\pi}_k}(s,\alpha)) &= -\mathbb{E}_{\tilde{a}\sim\hat{\pi}_{k+1}(s,\alpha)}[\partial\ell_{\epsilon,\alpha}^\kappa(\tilde{r}(s,\tilde{a})+\gamma {v}^{\hat{\pi}_k}(\tilde{s}',\tilde{u})-{v}^{\hat{\pi}_k}(s,\alpha))]  \\
    &=-\mathbb{E}_{\tilde{a}\sim\hat{\pi}_{k+1}(s,\alpha)}[A^{\hat{\pi}_k}(s,\alpha,\tilde{a})]\leq 0.
    \end{align*}
    By convexity of $F_{s,\alpha}^{k}(x)$, we conclude that $\arg\min_x F_{s,\alpha}^{k}(x)\geq {v}^{\hat{\pi}_k}(s,\alpha)$ and moreover that:
    \[\mathcal{B}_{\epsilon,\kappa}^{\hat{\pi}_{k+1}}{v}^{\hat{\pi}_k}(s,\alpha) = \arg\min_x F_{s,\alpha}^{k}(x)\geq {v}^{\hat{\pi}_k}(s,\alpha)\]
    The rest follows by monotonicity of $\mathcal{B}_{\epsilon,\kappa}^{\hat{\pi}_{k+1}}$ and the fact that ${v}^{\hat{\pi}_{k+1}}$ is the fixed point of the contraction $\mathcal{B}_{\epsilon,\kappa}^{\hat{\pi}_{k+1}}$. Namely, we can inductively show that for all $T=1,\dots,\infty$, 
    \[(\mathcal{B}_{\epsilon,\kappa}^{\hat{\pi}_{k+1}})^T {v}^{\hat{\pi}_{k}} = \mathcal{B}_{\epsilon,\kappa}^{\hat{\pi}_{k+1}}(\mathcal{B}_{\epsilon,\kappa}^{\hat{\pi}_{k+1}})^{T-1} {v}^{\hat{\pi}_{k}}\geq \mathcal{B}_{\epsilon,\kappa}^{\hat{\pi}_{k+1}} {v}^{\hat{\pi}_{k}}\geq {v}^{\hat{\pi}_{k}}.\]
    This leads to our conclusion since:
    \[{v}^{\hat{\pi}_{k+1}} = \lim_{T\rightarrow \infty} (\mathcal{B}_{\epsilon,\kappa}^{\hat{\pi}_{k+1}})^T {v}^{\hat{\pi}_{k}} \geq {v}^{\hat{\pi}_{k}}.\]
\end{proof}

\subsection{Proof of Proposition~\ref{prop:mkv}}
\Propmkv*
\begin{proof}
For all the $(s_t,\alpha_t)$ encountered by running Algo.~\ref{alg:static-var-exec1}, we show only when action $a_t$ chooses $\arg\max_a q^{\bar{\pi}^*}(s_t,\alpha_t,a)$, the above objective achieves the largest value.

From the proof of Proposition~\ref{prop:quant-pibar}, we know 
$$
\max_a q^{\bar{\pi}^*}(s_t,\alpha_t,a)=v^{\bar{\pi}^*}(s_t,\alpha_t) = v^*(s_t,\alpha_t),
$$
where $v^*(s_t,\alpha_t)$ is the optimal $\alpha_t$-quantile value that can be achieved in state $s_t$. By the definition of $q^{\bar{\pi}^*}$ under policy $\bar{\pi}^*$
$$
q^{\bar{\pi}^*}(s,\alpha,a)= \mathrm{VaR}_\alpha[\tilde{r}(s,a)+\gamma v^{\bar{\pi}^*}(\tilde{s}',\tilde{u})].
$$
Therefore,
$$
v^*(s_t,\alpha_t) =v^{\bar{\pi}^*}(s_t,\alpha_t)\geq \mathrm{VaR}_{\alpha_t}[\tilde{r}(s_t,a_t) + \gamma v^{\bar{\pi}^*}(\tilde{s}_{t+1},\tilde{u})]
, ~~\forall a_t.
$$
Similar to the proof of Proposition~\ref{prop:fix-point}, when $\epsilon= 0$ and $\kappa= 0$, this implies that for all $a$
\[\nabla_y f(s_t,\alpha_t,a,y)\Big|_{y=v^*(s_t,\alpha_t)} \geq  \nabla_y f(s_t,\alpha_t,a,y)\Big|_{y=\bar{q}(s_t,\alpha_t,a)}=0\]
with $f(s,\alpha,a,y):= \mathbb{E}\Big[\partial l_\alpha\big(\tilde{r}(s,a)+\gamma v^{\bar{\pi}^*}(\tilde{s}',\tilde{u})-y\Big]$ and $\bar{q}(s,\alpha,a):=\mathrm{VaR}_{\alpha}[\tilde{r}(s,a) + \gamma v^{\bar{\pi}^*}(\tilde{s}',\tilde{u})]$. Given 
$\nabla_y f(s,\alpha,a,y) = - \mathbb{E}\Big[\partial l_\alpha\big(\tilde{r}(s,a)+\gamma v^{\bar{\pi}^*}(\tilde{s}',\tilde{u})-y\big)\Big]$, we get:
$$\mathbb{E}\Big[\partial l_\alpha\big(\tilde{r}(s_t,a)+\gamma v^{\bar{\pi}^*}(\tilde{s}_{t+1},\tilde{u})-v^{\bar{\pi}^*}(s_t,\alpha_t)\big)\Big] \leq 0, ~\forall a
$$
with the maximum only achieved when $a_t=\arg\max_a q^{\bar{\pi}^*}(s_t,\alpha_t,a)$.
\end{proof}

\section{Full Algorithm}
The full algorithm is summarized in Algo.~\ref{alg:full}. One detail on quantile value function update omitted in the main text is that instead of computing one-step quantile regression loss, we perform weighted multi-step quantile regression similar as the multi-step advantage estimation. For a state $s_t$, its target values can be estimated using multi-step rollout as
$$
v^{(\iota)}_{\mathrm{target}}(s_t,\tilde{u})=r_t+\gamma r_{t+1}+...+\gamma^\iota v(s_{t+\iota},\tilde{u}).
$$
Then the weighted quantle regression loss for $s_t$ is computed as
$$
\lambda\mathbb{E}_\alpha[l_\alpha(v^{(1)}_{\mathrm{target}}(s_t,\tilde{u})-v(s_t,\alpha))] + \lambda^2\mathbb{E}_\alpha[l_\alpha(v^{(2)}_{\mathrm{target}}(s_t,\tilde{u})-v(s_t,\alpha))] + ...
$$

\begin{algorithm}[tb]
  \caption{CVaR-VaR Policy Gradient}
  \label{alg:full}
  \begin{algorithmic}
    \STATE {\bfseries Input:} policy $\bar{\pi}_\theta(a|s)$, quantile value function $v_\phi(s,\alpha)$, number of quantiles $I$, trajectory length $T$, risk level $\alpha_0$,  $\gamma\in(0,1]$, $\lambda\in(0,1)$, $\omega\in(0,1)$. Policy learning rate $\mathrm{lr}_\theta$, value learning rate $\mathrm{lr}_\phi$. Total iteration $M$. Number of trajectories $N$
    \STATE Compute discretized quantile levels $\Lambda=\{\frac{1}{2} (\frac{i-1}{I}+\frac{i}{I})\}_{i=1}^I $
    \FOR{$m=1$ {\bfseries to} $M$}
    \STATE \textit{// \textcolor{blue}{Sample trajectories}}
    \STATE $B\leftarrow \emptyset$
    \FOR{$i=1$ {\bfseries to} $N$}
    \STATE $\tau \leftarrow \emptyset$, $s\leftarrow$ env.reset(), $\alpha \sim U[0,\alpha_0]$
    \STATE projection $\alpha$ to discretized quantile levels $\alpha\leftarrow\mathrm{proj}(\alpha,\Lambda)$
    \FOR{$t=0$ {\bfseries to} $T-1$}
    \STATE $a\sim {\bar\pi_\theta}(\cdot|s),~~z\leftarrow v_\phi(s,\alpha)$
    \STATE $r, s' = \mathrm{env.step}(a)$
    \STATE $\tau$.append($s,\alpha, a$)
    \STATE $z\leftarrow (z-r)/\gamma$
    \STATE $\alpha\leftarrow \min\Big\{\beta\Big|v_\phi(s',\beta)\geq z,\beta\in \Lambda\Big\}$
    \STATE $s\leftarrow s'$
    \ENDFOR
    \STATE $B$.append($\tau$)
    \ENDFOR
    \STATE Compute trajectory returns $\{R(\tau_i)\}_{i=1}^N$ for $\tau_i\in B$ with $\gamma$
    \STATE Compute multi-step advantage $\{\bar A(s_t,\alpha_t,a_t)\}_{t=0}^{T-1}$ for all $\tau\in B$ by Eq.~\ref{eq:mkv-adv-multistep} with $\gamma$ and $\lambda$
    \STATE \textit{// \textcolor{blue}{CVaR-PG, i.e. Eq.~\ref{eq:cvar-pg}}}
    \STATE $g_1= \frac{1}{\alpha N} \sum_{i=1}^N\Big (\mathbb{I}_{\{R(\tau_i)\leq \hat{q}_{\alpha_0}\}} (R(\tau_i)-\hat{q}_{\alpha_0})\sum_{t=0}^{T-1} \nabla_\theta \log\bar{\pi}_\theta(a_{i,t}|s_{i,t})\Big )$, $\hat{q}_{\alpha_0}=\mathrm{VaR}_{\alpha_0}[\{R(\tau_i)\}_{i=1}^N]$
    \STATE \textit{// \textcolor{blue}{VaR-PG}}
    \STATE $g_2=\frac{1}{N} \sum_{i=1}^N\sum_{t=0}^T \bar A(s_{i,t},\alpha_{i,t},a_{i,t}) \nabla_\theta \log \bar\pi_\theta(a_{i,t}|s_{i,t})$
    \STATE \textit{// \textcolor{blue}{Update policy function}}
    \STATE $\theta\leftarrow \theta + \mathrm{lr}_\theta ((1-\omega) g_1 +  \omega g_2\big)$
    \STATE \textit{// \textcolor{blue}{Update value function}}
    \FOR{$i=1$ {\bfseries to} $N$}
    \STATE Compute multi-step target $\{v^{(\iota)}_{\mathrm{target}}(s_t,\tilde{u})=r_t+\gamma r_{t+1}+...+\gamma^\iota v(s_{t+\iota},\tilde{u})\}_{\iota=1}^{T-1-t}$ for all $s_t\in \tau_i$ 
    \STATE Compute weighted quantile regression loss: 
    \STATE $g_3(s_t)= \lambda \mathbb{E}_{\alpha\in\Lambda}[l_\alpha(v^{(1)}_{\mathrm{target}}(s_t,\tilde{u})-v_\phi(s_t,\alpha)] + \lambda^2 \mathbb{E}_{\alpha\in \Lambda}[l_\alpha(v^{(2)}_{\mathrm{target}}(s_t,\tilde{u})-v_\phi(s_t,\alpha))]+...$ for all $s_t\in \tau_i$
    \STATE $\phi \leftarrow \phi - \mathrm{lr}_\phi  \mathbb{E}_{s\sim \tau_i }[\nabla_\phi g_3(s)]$
    \ENDFOR
    \ENDFOR
  \end{algorithmic}
\end{algorithm}

\section{Experiments Details}
 \label{sec:exp}
 \subsection{Network Architecture}
 All policy networks and value networks (only CVaR-PG does not require a value function) are implemented using deep neutral networks with two hidden layers. For discrete action, softmax is applied to the output layer of the policy to create action probabilities. For continuous action, the network learns the mean and std of a normal distribution.
 
 PCVaR-PG~\cite{kim2024risk} requires to learn functions taking cumulative rewards so far as input. We normalize this value to the range in $[0,1]$ and use cosine embedding as used in IQN~\cite{dabney2018implicit} to embed this value, i.e., denote the input of this value as $x$, denote the input size of the cosine embedding layer as $n$, denote the parameters in this linear layer as $a_{ij}$ and $b_{j}$, the $j$-th output is
 $$
 \phi_j(x)=\mathrm{ReLU}(\sum_{i=0}^{n-1}\mathrm{cos}(\pi i x) a_{ij}+b_j)
 $$

 Our method requires to learn the quantile value function of a state. We follow QR-DQN~\cite{dabney2018distributional} to output $I$ values, corresponding to quantile levels $\{\frac{1}{2} (\frac{i-1}{I}+\frac{i}{I})\}_{i=1}^I$. To ensure the monotonicity of the quantile function, the first output value is treated as the base quantile value, and the remaining $I-1$ outputs are increasing deltas. We apply \texttt{nn.functional.softplus} to make delta non-negative.

 In Maze domain, the input to both policy and value nets are $(x,y)$ coordinates. To embed the coordinate, we first transform $(x,y)$ to a scalar index and then use \texttt{nn.Embedding} in PyTorch to embed this index.

 As mentioned, our method uses a multi-step advantage estimation. We also implement this technique for the policy learning in PCVaR-PG~\cite{kim2024risk} and RET-CAP~\cite{mead2025return}. 
 
 \subsection{Maze}
 \textbf{Learning parameters}. Discount factor $\gamma=0.999$. Multi-step advantage estimation $\lambda=0.95$ (in PCVaR-PG, RET-CAP, and CVaR-VaR). Optimizer is Adam. The embeding size of \texttt{nn.Embedding} is 16. Hidden size of neutral network is 64. 
 
 Policy learning rate is selected from $\{$7e-4, 5e-4, 3e-4, 1e-4, 7e-5$\}$. Value learning rate is selected from $\{$1, 2, 5, 10$\}$ times the policy learning rate. The main learning parameters are summarized in Table~\ref{tab:maze-param}. 
 
 RET-CAP~\cite{mead2025return} considers $\gamma=1$ since only when $\gamma=1$, the summation of modified reward $\hat{r}_t=\min(k_t, q^*_\alpha)-\min(k_{t-1},q^*_\alpha)$ equals $\min(R(\tau),q^*_\alpha)$. So we set $\gamma=1$ and $q^*_\alpha=0$. PCVaR-PG~\cite{kim2024risk} requires to estimate $q^*_\alpha$, we also set this value to stabilize training, though we find the learning curves show little difference if estiamting this value during learning. The trade-off $\omega=0.5$ in CVaR-VaR.

 \begin{table}
  \caption{Learning parameters in Maze.}
  \label{tab:maze-param}
  \begin{center}
    \begin{small}
      \begin{sc}
        \begin{tabular}{lcccccc}
          \toprule
          Method  & $\pi$ lr         & $V$ lr      & $f(s,k)$ lr & Normalize Adv? & n\_quantile \\
          \midrule
          REINFORCE    & 7e-4 & 7e-4 & - & $\times$ & - \\
          CVaR-PG      & 5e-4 &   -  & - & - & -\\
          CVaR-VaR     & 5e-4 & 5e-4 & - & $\times$ & 10  \\
          RET-CAP     & 5e-4 & 5e-3 &  -& $\surd$ & -\\
          \bottomrule
        \end{tabular}
      \end{sc}
    \end{small}
  \end{center}
  \vskip -0.1in
\end{table}

\subsection{LunarLander}
We use the discrete action version of LunarLander. The state dimension is 8, the action space is 4. A detailed description is available at this webpage\footnote{https://gymnasium.farama.org/environments/box2d/lunar\_lander/}.

\textbf{Learning parameters}. Discount factor $\gamma=0.999$. Multi-step advantage estimation $\lambda=0.95$ (in PCVaR-PG, RET-CAP, and CVaR-VaR). Optimizer is Adam. Hidden size of neural network is 128.
 
Policy learning rate is selected from $\{$7e-4, 5e-4, 3e-4, 1e-4, 7e-5$\}$. Value learning rate is selected from $\{$1, 2, 5, 10$\}$ times the policy learning rate. The main learning parameters are summarized in Table~\ref{tab:ll-param}. 

 \begin{table}
  \caption{Learning parameters in LunarLander.}
  \label{tab:ll-param}
  \begin{center}
    \begin{small}
      \begin{sc}
        \begin{tabular}{lcccccc}
          \toprule
          Method  & $\pi$ lr         & $V$ lr      & $f(s,k)$ lr & Normalize Adv? & n\_quantile \\
          \midrule
          REINFORCE    & 7e-4 & 7e-3 & - & $\times$ & -  \\
          CVaR-PG      & 5e-4 &   -  & - & - & -\\
          CVaR-VaR     & 5e-4 & 5e-4 & - & $\surd$ & 10  \\
          PCVaR-PG   & 5e-4 & 5e-4& 5e-4& $\surd$  & -        \\
          RET-CAP     & 5e-4 & 1e-3 &  -& $\surd$ & -\\
          MIX & 3e-4 & - & - & - & -\\
          \bottomrule
        \end{tabular}
      \end{sc}
    \end{small}
  \end{center}
  \vskip -0.1in
\end{table}

 \begin{table}
  \caption{Learning parameters in Inverted Pendulum.}
  \label{tab:ivp-param}
  \begin{center}
    \begin{small}
      \begin{sc}
        \begin{tabular}{lcccccc}
          \toprule
          Method  & $\pi$ lr         & $V$ lr      & $f(s,k)$ lr & Normalize Adv? & n\_quantile \\
          \midrule
          REINFORCE    & 3e-4 & 3e-3 & - & $\times$ & -  \\
          CVaR-PG      & 3e-4 &   -  & - & - & -\\
          CVaR-VaR     & 3e-4 & 3e-4 & - & $\surd$ & 10  \\
          PCVaR-PG   & 3e-4 & 3e-4& 3e-4& $\surd$  & -        \\
          RET-CAP     & 3e-4 & 3e-4 &  -& $\surd$ & -\\
          MIX & 3e-4 & - & - & - & -\\
          \bottomrule
        \end{tabular}
      \end{sc}
    \end{small}
  \end{center}
  \vskip -0.1in
\end{table}

For RET-CAP~\cite{mead2025return}, $\gamma=1$. $q^*_\alpha=270\in[230, 250, 270, 290]$. The value of $q^*=270$ is also suggested by the authors. For CVaR-VaR, $\omega=0.5$ for the first $40\%$ percent of iterations and linearly decay to $0$ afterwards. For MIX~\cite{luo2024simple}, the training parameters for its risk-neutral component IQL~\cite{kostrikov2022offline} follow the original paper. 

\subsection{Inverted Pendulum}
The state dimension is 4. The one dimension action is continuous in range $[-3,3]$. A detailed description is available at this webpage\footnote{https://gymnasium.farama.org/environments/mujoco/inverted\_pendulum/}.

\textbf{Learning parameters}. Discount factor $\gamma=0.999$. Multi-step advantage estimation $\lambda=0.95$ (in PCVaR-PG, RET-CAP, and CVaR-VaR). Optimizer is Adam. Hidden size of neural network is 128.
 
Policy learning rate is selected from $\{$7e-4, 5e-4, 3e-4, 1e-4, 7e-5$\}$. Value learning rate is selected from $\{$1, 2, 5, 10$\}$ times the policy learning rate. The main learning parameters are summarized in Table~\ref{tab:ivp-param}.

For RET-CAP~\cite{mead2025return}, $\gamma=1$. $q^*_\alpha=230\in[230, 250, 270, 290]$. For PCVaR-PG~\cite{kim2024risk}, $q^*_\alpha=270\in[230, 250, 270, 290]$. For CVaR-VaR, $\omega=0.5$ for the first $25\%$ percent of iterations and is $0$ afterwards. For MIX~\cite{luo2024simple}, the training parameters for its risk-neutral component IQL~\cite{kostrikov2022offline} follow the original paper. 




\end{document}